\documentclass{article}

\usepackage[numbers,compress]{natbib}

\usepackage{custom_template}
\usepackage{siunitx}
\usepackage{booktabs}

\usepackage{amsmath,amsfonts,bm}

\def\eqref#1{equation~\ref{#1}}

\def\1{\bm{1}}

\DeclareMathAlphabet{\mathsfit}{\encodingdefault}{\sfdefault}{m}{sl}
\SetMathAlphabet{\mathsfit}{bold}{\encodingdefault}{\sfdefault}{bx}{n}

\usepackage{hyperref}
\usepackage{url}
\usepackage{algorithm}   
\usepackage{stmaryrd}
\usepackage{algpseudocode}
\usepackage{graphicx}
\usepackage{wrapfig}
\usepackage{comment}
\usepackage{makecell}
\usepackage{enumitem}
\usepackage{multirow}
\usepackage{xcolor}
\usepackage{comment}
\usepackage{subcaption}
\usepackage{ntheorem,lipsum}
\newtheorem{remark}{Remark}
\newtheorem{theorem}{Theorem}
\newtheorem{lemma}{Lemma}
\theorembodyfont{\upshape}
\newtheorem{definition}{Definition}
\newtheorem{example}{Example}
\newtheorem{corollary}{Corollary}

\newtheorem{proof}{Proof}

\newtheorem{proposition}{Proposition}

\newcommand{\method}{\textsc{TensorGRaD}}

\title{\method: Tensor Gradient Robust Decomposition for Memory-Efficient Neural Operator Training} %

\author{Sebastian Loeschcke$^{1}$, David Pitt$^{2}$,  Robert Joseph George$^{2}$, Jiawei Zhao$^{3}$, Cheng Luo$^{2}$  \textbf{Yuandong Tian}$^{3}$, \textbf{Jean Kossaifi}$^{4}$, \textbf{Anima Anandkumar}$^{2}$ \\
$^{1}$University of Copenhagen $^{2}$California Institute of Technology, $^{3}$Meta FAIR, $^{4}$NVIDIA AI
}

\begin{document}

\maketitle

\begin{abstract}
Scientific problems require resolving multi-scale phenomena across different resolutions and learning solution operators in infinite-dimensional function spaces. 
Neural operators provide a powerful framework for this, using tensor-parameterized layers to capture complex, multi-dimensional relationships. However, scaling neural operators to high-resolution problems leads to significant computational demands, making the training of industrial-scale models prohibitive.
In this work, we introduce \textbf{\method}, a novel method that directly addresses the memory challenges associated with optimizing large tensor-structured weights.
Our approach, based on a \textit{robust tensor decomposition}, factorizes gradients as the sum of a low-rank tensor and a sparse one to efficiently capture information within optimizer states, including outliers.
Additionally, we provide a recipe for mixed precision training of \method, achieving further memory savings without sacrificing accuracy.
We showcase the effectiveness of \method\ on Fourier Neural Operators, a class of models crucial for solving partial differential equations (PDE).
We provide theoretical guarantees for \method\, demonstrating its fundamental advantage over matrix-based gradient compression methods.
We empirically demonstrate large improvements across various PDE tasks, including the challenging turbulent Navier-Stokes case at a Reynolds number of $10^5$. \method\ reduces total memory usage by over 50\% while maintaining and sometimes even improving accuracy. 
\end{abstract}

\section{Introduction}

Modern deep learning has shifted towards large-scale foundation models, which have enabled unprecedented performance across diverse domains such as natural language processing, computer vision, and scientific computing~\cite{Brown2020,Kirillov2023}.
This represents a paradigm shift from traditional machine learning, where performance improvements are driven by scaling laws—requiring increases in data, compute, and model size~\cite{xiao2025rethinkingconventionalwisdommachine}.
This scaling comes at the cost of growing memory requirements. Adaptive optimizers such as Adam~\cite{kingma2014adam}, while crucial in training these large models, worsen this issue by storing additional moment tensors (e.g., first and second order moments for Adam) for each weight, which significantly increases the memory overhead~\cite{zhao2024galore,2024grass,LoQT}

This memory requirement is exacerbated in the case of scientific computing, both by the size and the nature of the data and models involved. Solving scientific problems typically involves solving partial differential equations (PDEs) and resolving multi-scale phenomena, on very large-scale data~\cite{nature_no}.
This multi-dimensional data is naturally represented using tensors: multidimensional arrays that offer a natural framework for representing and manipulating complex, high-dimensional
 data structures~\cite{siam}.
For instance, in weather forecasting, data can span spatial grids, time steps, and atmospheric variables, leading to high-order tensor representations~\cite{bonev2023sphericalfourierneuraloperators}.

\begin{figure}
    \centering
    \includegraphics[width=0.92\linewidth]{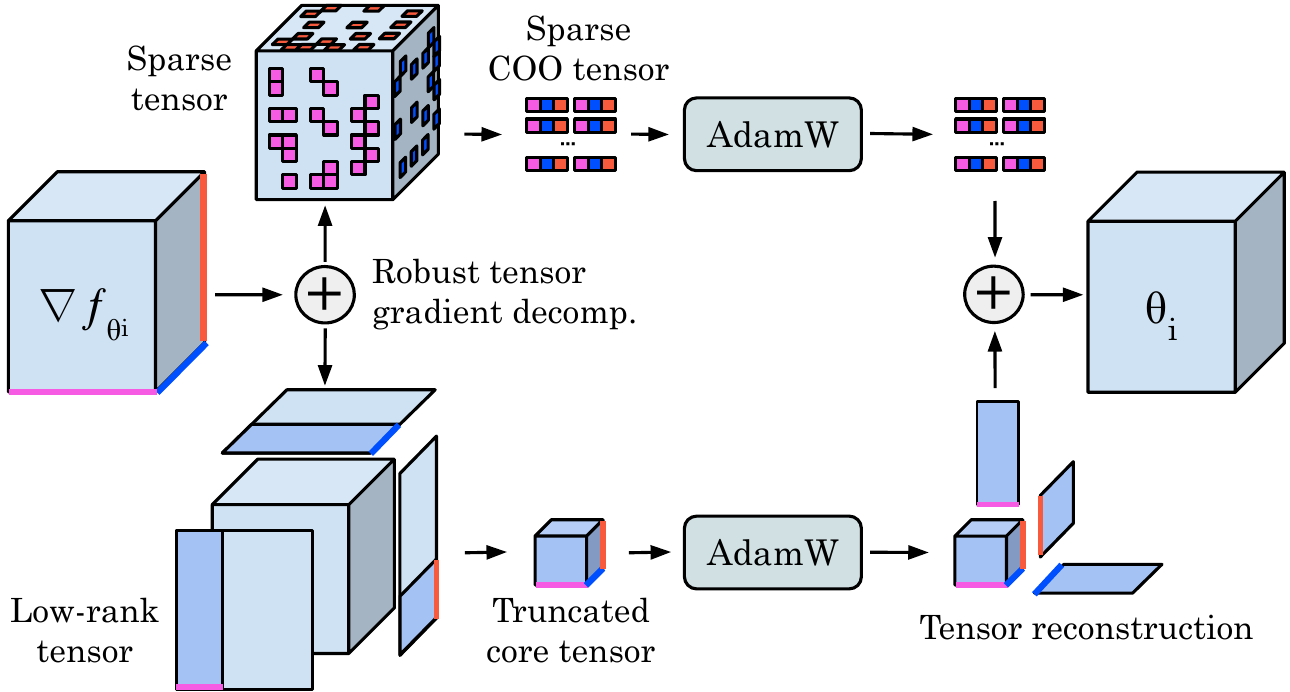}
    \caption{\textbf{Overview of \method}. Low-rank plus sparse decomposition}
    \label{fig:main1}\vspace{-10pt}
\end{figure}%

\textbf{Neural operators} have been proposed as the natural framework to tackle these problems, generalizing deep learning from learning in finite-dimensional spaces to learning in function spaces~\cite{li2023geometryinformedneuraloperatorlargescale}. Unlike neural networks, neural operators learn a mapping between function spaces, making them naturally suited for capturing the multi-scale structure of scientific data. 
To capture these multi-scale relationships, Neural Operators leverage the inherent (multi-dimensional) structure in the data, which requires maintaining high-order tensor weights and gradients to capture complex spatial, temporal, and channel interactions. As a result, unlike typical models in natural language processing or computer vision, where memory is dominated by activations, the memory overhead in neural operators is primarily driven by the tensor-structured weights and gradients. 
This memory requirement grows intractably with the scale of the data and has hindered the scaling of neural operators to large and complex scientific problems.

While many recent methods reduce optimizer memory in large language models, they are not directly applicable to neural operators due to the multi-dimensional structure of their weights and gradients. 
GaLore (Gradient Low-Rank Projection)~\cite{zhao2024galore}, for example, uses a Singular Value Decomposition (SVD) to compute low-rank approximations of gradient matrices before computing optimizer states, significantly reducing memory usage during training. Extending GaLore to neural operators requires flattening gradient tensors into matrices, which disrupts their multi-dimensional structure, discarding important relationships between modes (e.g., spatial, temporal, or channel interactions). The resulting flattened matrix is not naturally low-rank and hence, standard GaLore has poor performance when low-rank structures are enforced.

This not only leads to a suboptimal low-rank approximation but can also degrade model performance, especially in scientific applications where these interactions are crucial.
Alternatives to low-rank gradient projections include GRASS~\cite{2024grass}, which uses structured sparse projections that match low-rank methods at higher memory budgets but underperform under strict constraints. 
Generally, existing methods primarily focus on either low-rank or sparse representations, rarely exploring their combined application to tensor gradients, where multi-dimensional structure is critical.

\textbf{In this work}, we propose \textbf{\method}, a novel method for efficient training of neural operators that directly addresses the memory challenges associated with tensor-structured gradients. Our approach hinges on a \textit{robust tensor decomposition (RTD)}~\cite{NIPS2014_robust_tensor_decomposition} of the gradients during optimization. 
Specifically, we generalize both low-rank and sparse projections to tensors and unify them in a \emph{robust tensor decomposition} framework. We prove analytically that a direct extension of GaLore, relying on matricizing the gradient tensors, fails to preserve the multilinear structure required by Neural Operators. We also verify this empirically in ablation studies.

Our robust tensor decomposition framework is able to
accurately compress gradients $\mathcal{G}$ by decomposing them as a sum $\mathcal{G} = \mathcal{L} + \mathcal{S}$ of a low rank tensor approximation $\mathcal{L}$ with a sparse part $\mathcal{S}$.
We demonstrate that our robust tensor gradient factorization remains stable under a mixed-precision strategy, running activations, weights, and gradients in half precision while maintaining optimizer states in full precision. This setup achieves substantial memory savings without compromising model accuracy. Empirically, we show that using half-precision optimizer states degrades performance, underscoring the importance of our method for preserving gradient information.

Implemented with AdamW, \method\ reduces memory usage by up to $75\%$ for high-resolution neural operator learning, while matching or improving baseline accuracy across several PDE benchmarks. On the challenging Navier–Stokes at $1024 \times 1024$ resolution with a Reynolds number of $10^5$, where turbulent structures emerge across multiple scales, our mixed precision \method\ matches the test $L_2$ loss of the full-precision Adam optimizer, while reducing optimizer memory usage by up to 75\% and cutting and total memory cost of more than 55\% while matching or even improving on baseline performance. Code is available at 
\textcolor{red!60!yellow}{ \url{https://github.com/neuraloperator/tensorgrad}}

\begin{figure}
    \centering
    \begin{subfigure}[b]{0.44\textwidth}
         \includegraphics[width=\linewidth]{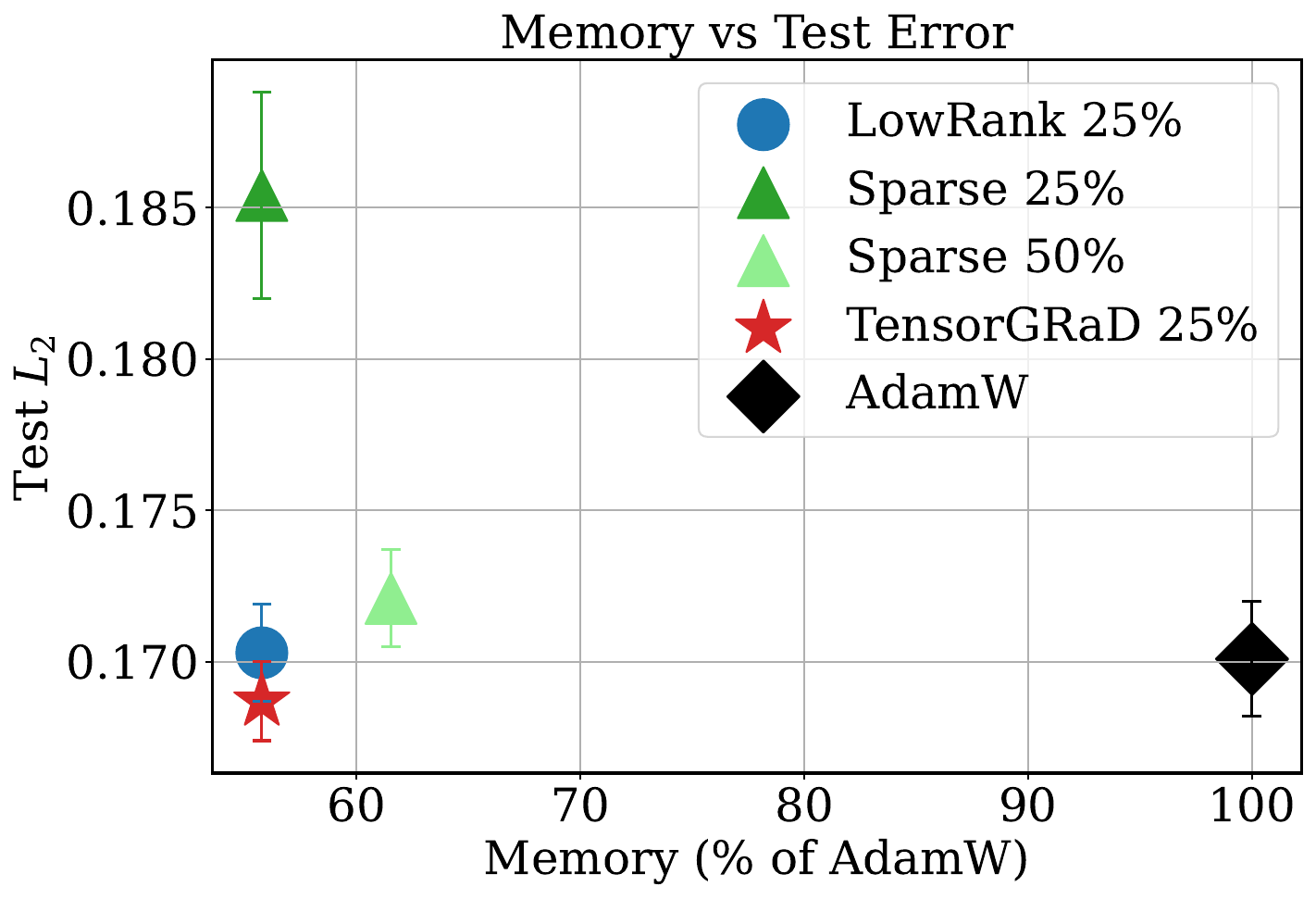}
    \end{subfigure}
    \hfill
    \begin{subfigure}[b]{0.52\textwidth}
        \includegraphics[width=\linewidth]{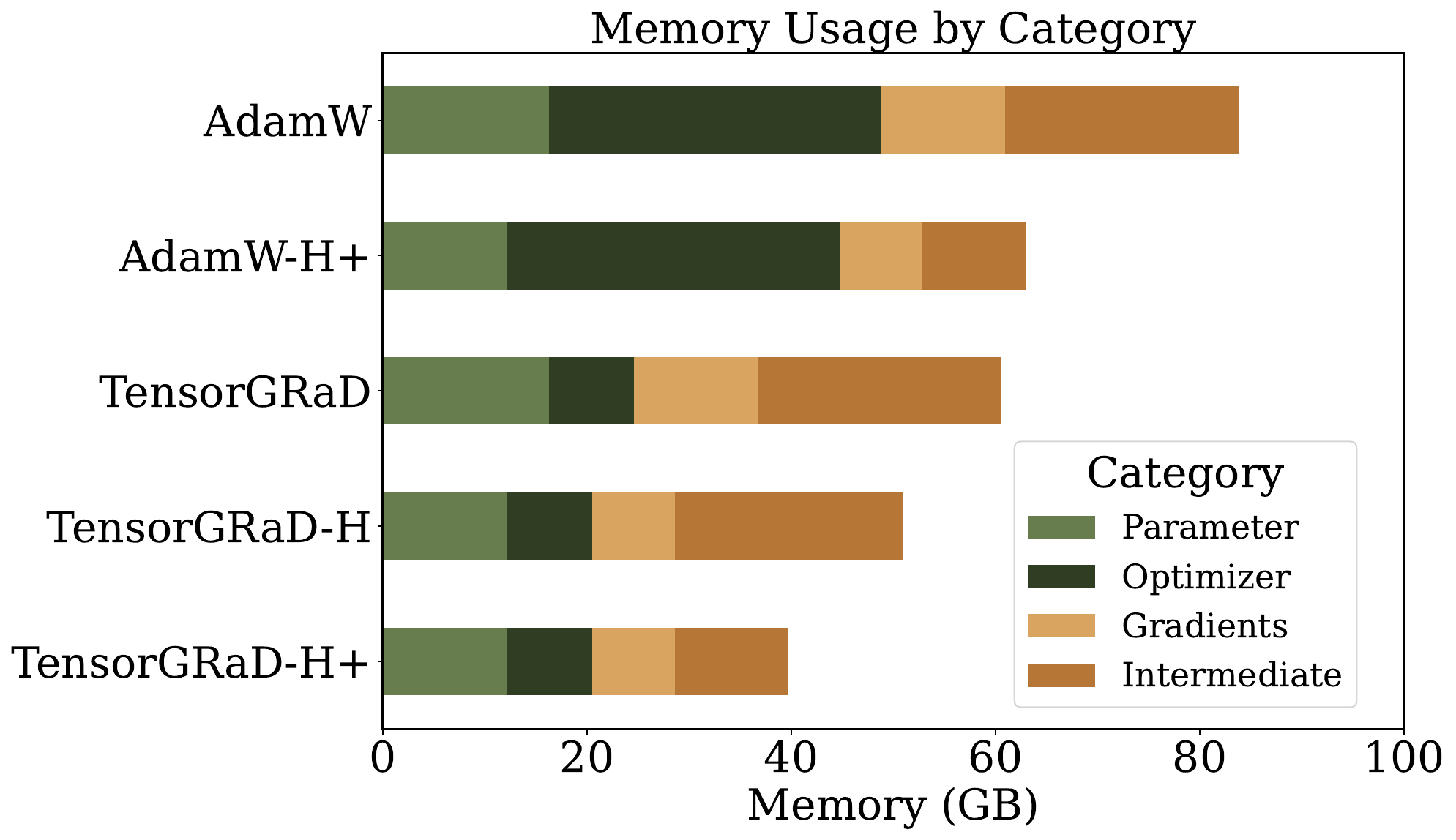}
    \end{subfigure}
    \caption{\textbf{Memory and performance.} 
Left: Comparison of low-rank, structured sparse, and \method\ (mixed precision) vs.\ Adam on Navier–Stokes $1024\times1024$. \method\ offers the best memory–accuracy trade-off. 
Right: Peak CUDA memory for FNO models with 256 channels. \textbf{H}: uses half precision for weights/gradients and full for optimizer states, and \textbf{+} includes activation checkpointing.}
    
    \label{fig:memuse_comparison}
\vspace{-5mm}
\end{figure}

\section{\textbf{\method}}\vspace{-10pt}
\label{sec:method}
In this section, we first introduce the necessary background before going into detail in our method, illustrated in Fig.~\ref{fig:main1}, its training, implementation, and theoretical properties.

\subsection{Background: Tensors and Neural Operators }
\textbf{Tensors} are multidimensional arrays that generalize the concepts of vectors (first-order tensors) and matrices (second-order tensors) to higher orders. An $N$th-order tensor $\mathcal{X} \in \mathbb{C}^{I_1 \times I_2 \times \cdots \times I_N}$ is an $N$-way array where each mode $n$ has dimension $I_n$.

\textbf{Neural Operators} $\mathcal{G}_\theta : \mathcal{A} \times \theta \to \mathcal{U}$ combine linear integral operators $\mathcal{K}$ with pointwise non-linear activations $\sigma$ to approximate non-linear operators, mapping initial conditions $a \in \mathcal{A}$ to solutions $u \in \mathcal{U}$. Their operation is defined as $\mathcal{G}_\theta := \mathcal{Q} \circ (W_L + \mathcal{K}_L) \circ \cdots \circ \sigma(W_1 + \mathcal{K}_1) \circ \mathcal{P}$, where $\mathcal{P}$ and $\mathcal{Q}$ are pointwise neural networks for encoding and decoding, $W_l$ are linear operators, $\mathcal{K}_l$ are integral kernel operators, and $\sigma$ are activation functions. 

The \textbf{Fourier Neural Operator (FNO)} proposes a specific convolution operator for $\mathcal{K}$, defined as $(\mathcal{K}v_l)(x) = \mathcal{F}^{-1}(R \cdot T_K \mathcal{F}v_l)(x)$, where $\mathcal{F}$ and $\mathcal{F}^{-1}$ are the Fourier transform and its inverse, $R$ is a learnable transformation, and $T_K$ truncates to the lowest $K$ Fourier modes. This formulation allows FNO to be discretization-invariant, producing high-quality solutions for query points not in the training grid and enabling transfer between different grid resolutions and discretizations.

\subsection{\method}
\textbf{Robust decomposition of gradients.} Our method hinges on a robust decomposition of the gradient during training.
Instead of minimizing a reconstruction error under some Gaussian error assumption, robust decomposition decomposes inputs as a low-rank part and a sparse part, even when it contains gross corruption or outliers~\cite{robust_pca}.
Robust Tensor Decomposition (RTD)~\cite{NIPS2014_robust_tensor_decomposition} extends this concept to higher-order tensors, proving that low-rank components can be separated from sparse corruptions using convex optimization. Later work~\cite{yang2017unifiedframework} provided convergence guarantees for low-rank plus sparse recovery in matrices.

Leveraging these advances in robust decomposition, we transform the gradients into compressed representations before computing and storing their optimizer moments. Specifically, we use two complementary forms of structure in gradient tensors: unstructured sparsity and low-rank decompositions. Each form compresses a distinct aspect of the gradient: sparse representations preserve sharp, localized signals, while low-rank approximations model smooth, global structure. These transformations are applied sequentially and combined to form a \textit{robust tensor decomposition}.

\paragraph{Unstructured sparse gradient tensor.}
We represent localized gradient information using a sparse COO-format tensor $\hat{\mathcal{G}}_S$ supported on a fixed set of indices $\Omega \subseteq [I_1] \times \cdots \times [I_N]$. This index set is constructed by selecting $k = \lceil \rho I \rceil$ entries from $\mathcal{G} \in \mathbb{C}^{I_1 \times \cdots \times I_N}$ according to a sparsification strategy, e.g., by inspiration from GRASS~\cite{2024grass} top-$k$ magnitude, probabilistic sampling, or uniform random selection. The corresponding values are extracted to define $\hat{\mathcal{G}}_S = \operatorname{Sparse}(\mathcal{G}, \Omega),$
where $\hat{\mathcal{G}}_S$ is a $k$-nonzero sparse tensor in COO format, consisting of index–value pairs. 
The sparse index set $\Omega$ is recomputed only every $T$ steps and reused in between, while the sparse tensor $\hat{\mathcal{G}}_S$ is extracted from the current gradient at every step.

This format is compatible with standard sparse tensor operations, enabling direct addition, scaling, and indexing without reconstructing a dense tensor. 
Overall, this representation requires storing $k$ integer indices and $k$ complex values. It supports efficient computation in the sparse format, such as gather and scatter operations, with no dense intermediates.

\paragraph{Low-rank gradient tensor decomposition.}
To compress high-dimensional gradient tensors, we use a Tucker decomposition~\cite{Tucker1966SomeMN,kolda2009tensor}, a higher-order generalization of low-rank matrix factorization. Given a tensor $\mathcal{G} \in \mathbb{C}^{I_1 \times \cdots \times I_N}$, we approximate it as
$
\mathcal{G} \approx \llbracket \mathcal{C};\, U^{(1)}, \dots, U^{(N)} \rrbracket,$
where $\mathcal{C}$ is a core tensor of size $\mathbb{C}^{r_1 \times \cdots \times r_N}$ and $U^{(n)} \in \mathbb{C}^{I_n \times r_n}$ are orthonormal factor matrices.
We compute the decomposition once and discard the core, retaining only the factor matrices. These are then reused to compress incoming gradients into a factorized representation:
$
\hat{\mathcal{G}}_L = \mathcal{G} \times_1 {U^{(1)}}^\top \cdots \times_N {U^{(N)}}^\top$.
Optimizer states are maintained directly on $\hat{\mathcal{G}}_L$, and the transformed tensor is reconstructed after the update via:
$\tilde{\mathcal{G}}_L = \hat{\mathcal{G}}_L \times_1 U^{(1)} \cdots \times_N U^{(N)}.$

This decomposition reduces memory by maintaining only the factor matrices (each of size $\mathbb{C}^{I_n \times r_n}$) and the compressed optimizer state.
It offers three key properties central to our method:
\begin{itemize}[topsep=1pt,itemsep=1pt, left=1pt]
\setlength\itemsep{0em}
\item \textbf{SVD generalization:} In the special case of $N=2$, the Tucker decomposition reduces to the standard matrix SVD, linking our method naturally to GaLore.
\item \textbf{Orthonormality and efficiency:} The factor matrices $U^{(n)}$ are orthonormal, allowing stable compression via mode-wise multiplication with $U^{(n)\top}$, and reconstruction using $U^{(n)}$ directly without requiring matrix inversion.
\item \textbf{Structure preservation:} Tucker factorization maintains mode-wise information, avoiding the loss of semantic structure associated with tensor flattening and Kronecker approximations.
\end{itemize}

\paragraph{Residual and composition.}
The two components are applied sequentially. After forming the sparse or low-rank approximation $\tilde{\mathcal{G}}_1$, we compute the residual $
\mathcal{R} = \mathcal{G} - \tilde{\mathcal{G}}_1.
$
We then use $\mathcal{R}$ to compute $\tilde{\mathcal{G}}_2$ instead of ${\mathcal{G}}$, i.e., after computing the residual either the low-rank decomposition or the sparse tensor is computed on $\mathcal{R}$.
 This composition allows each branch to focus on distinct parts of the tensors. Using sparse top-k removal first can remove high-frequency components or outliers. Alternatively, using low-rank first can capture the smoothness, followed by a sparse projector that can capture the most significant part not converted by the low-rank tensors.

\paragraph{Optimizer update.}
Each component is updated independently using Adam in its compressed space. First and second moment estimates $(\mathcal{M}_S, \mathcal{V}_S)$ and $(\mathcal{M}_L, \mathcal{V}_L)$ are maintained for the sparse and low-rank parts, respectively. The full update is reconstructed as:
$$
\Delta \mathcal{W} = \alpha \left( \tilde{\mathcal{G}}_L + \lambda\,\tilde{\mathcal{G}}_S \right),
\qquad
\mathcal{W}_{t+1} = \mathcal{W}_t + \eta \cdot \Delta \mathcal{W}.
$$
The order of the decompositions matters only during the forward pass, as the first component defines the residual for the second. For memory efficiency, the low-rank component is reconstructed first, and the sparse values are added directly into the same tensor via scatter operations. This way, we avoid having two full tensors in memory at once.
This robust decomposition reduces memory overhead by maintaining compact optimizer states in compressed formats. Unstructured sparsity captures fine-grained outliers, while mode-wise low-rank decompositions preserve global structure. By combining them, \method\ achieves higher fidelity under strong memory constraints than either approach alone. We present the pseudocode in~\ref{pseudocode}.

Fig.~\ref{fig:robust_and_tradeoff} compares the distribution of reconstruction errors for a complex gradient tensor ($64^3 \times 32$) from an FNO layer. We evaluate three strategies: (1) unstructured sparse followed by low-rank compression ($5\% + 20\%$), (2) low-rank followed by sparse ($20\% + 5\%$), and (3) pure low-rank compression at $25\%$. The reconstruction error is measured as the absolute difference between the original gradient $\mathcal{G}$ and its approximation $\tilde{\mathcal{G}}$, i.e., $|\mathcal{G} - \tilde{\mathcal{G}}|$. All layers exhibit similar trends. The pure low-rank method introduces more high-magnitude outliers, while the combined strategies produce tighter distributions. Applying sparse compression first yields fewer large errors but slightly higher average reconstruction error, suggesting a trade-off between error spread and overall magnitude.

\subsection{Mixed precision training}
\label{sec:mixed_precision}\vspace{-5pt}
We show that our gradient compression method remains stable when activations, weights, and gradients are computed in half precision, provided that optimizer states are maintained in full precision. This extends recent work on FNO training by Tu et al.~\cite{tu2024guaranteedapproximationboundsmixedprecision}, which established approximation guarantees for mixed-precision FNO training using AMP. However, their approach retains weights in full precision and casts them to half precision during computation. In contrast, we explore full half-precision training beyond AMP, including weight storage. Additionally, we find that storing optimizer states in half precision significantly degrades performance, further emphasizing projected optimizer states as an alternative. Combining mixed-precision training with \method\ allows us to achieve further memory savings without sacrificing performance.

\subsection{\textbf{Implementation}}\vspace{-5pt}
The \textbf{low-rank component} of \method\ uses the efficient Tucker decomposition from TensorLy~\cite{kossaifi2019tensorly}, implemented via Higher-Order Orthogonal Iteration. See Appendix~\ref{sec:impl_hoi} for more details. All subsequent operations, like compressing the gradients and reconstructing the low-rank updates, are performed using PyTorch. The \textbf{unstructured sparse component} is implemented natively in PyTorch. We extract the top-$k$ or randomly sampled values based on a given sparsification strategy and store them as index–value pairs. This format supports direct operations like elementwise scaling and addition without dense reconstruction. Together with gradient compression and mixed precision training, these techniques allow \textbf{\method} to scale efficiently to large neural operators.
We refer to this full setup as \textbf{\method+}, and highlight its most memory-efficient in Fig.~\ref{fig:memuse_comparison}.

\begin{algorithm}
\caption{\textsc{\method}: Adam with Sparse and Low-Rank Gradient Compression}
\label{pseudocode}
\begin{algorithmic}[1]
\Require Weight tensor $\mathcal{W} \in \mathbb{C}^{N_1 \times N_2 \times N_3 \times N_4}$. Step size $\eta$, scale factor $\alpha$, decay rates $\beta_1, \beta_2$, rank $r$, sparsity $\rho$, sparsity scale factor $\lambda$, subspace change frequency $T$. 
\State Initialize step $t \gets 0$
\State Initialize first-order moments $\mathcal{M}_L \in \mathbb{C}^{r \times r \times r \times r} \gets 0$, $\mathcal{M}_S \in \mathbb{C}^{\rho N_1 \times \rho N_2 \times \rho N_3 \times \rho N_4} \gets 0$
\State Initialize second-order moments $\mathcal{V}_L \in \mathbb{C}^{r \times r \times r \times r} \gets 0$, $\mathcal{V}_S \in \mathbb{C}^{\rho N_1 \times \rho N_2 \times \rho N_3 \times \rho N_4} \gets 0$
\Repeat
    \State $\mathcal{G}_t \gets -\nabla_{\mathcal{W}} \phi_t(\mathcal{W}_t)$
    
    \If{$t \bmod T = 0$}
        \State $\Omega \gets \operatorname{SparseIndices}(\mathcal{G}_t, \rho, \text{strategy})$ \Comment{Select new sparse indices}
        \State $\mathcal{R}_L \gets \mathcal{G}_t - \operatorname{Sparse}(\mathcal{G}_t, \Omega)$ \Comment{Subtract sparse tensor from input}
        \State $\{U^{(n)}\}_{n=1}^4 \gets \operatorname{TuckerFactors}(\mathcal{R}_L, r)$ \Comment{Update low-rank factors}
    \EndIf
    
    \State $\hat{\mathcal{G}}_S \gets \operatorname{Sparse}(\mathcal{G}_t, \Omega)$ \Comment{Create sparse tensor at current step}
    \State $\hat{\mathcal{G}}_L \gets (\mathcal{G}_t - \hat{\mathcal{G}}_S) \times_1 {U^{(1)}}^\top \cdots \times_4 {U^{(4)}}^\top$ \Comment{Low-rank compression}
    
    \State $\hat{\mathcal{G}}_S \gets \textsc{AdamUpdate}(\hat{\mathcal{G}}_S, \mathcal{M}_S, \mathcal{V}_S, \beta_1, \beta_2, t)$
    \State $\hat{\mathcal{G}}_L \gets \textsc{AdamUpdate}(\hat{\mathcal{G}}_L, \mathcal{M}_L, \mathcal{V}_L, \beta_1, \beta_2, t)$

    \State $\tilde{\mathcal{G}} \gets \alpha \cdot (\hat{\mathcal{G}}_L \times_1 U^{(1)} \cdots \times_4 U^{(4)})$ \Comment{Reconstruct low-rank part}
    \State $\tilde{\mathcal{G}} \gets \tilde{\mathcal{G}} + \lambda \cdot \hat{\mathcal{G}}_S$ \Comment{Add scaled sparse tensor to low-rank tensor}

    \State $\mathcal{W}_{t+1} \gets \mathcal{W}_t + \eta \cdot \tilde{\mathcal{G}}$
    \State $t \gets t + 1$
\Until{convergence}
\end{algorithmic}
\end{algorithm}

\subsection{Theoretical Results of \textbf{\method}}\vspace{-5pt}
We extend the theoretical foundations of GaLore~\cite{zhao2024galore} to tensor-structured weights, proving both convergence guarantees and low-rank emergence during training. Our analysis shows that gradients of FNO models naturally develop low-rank structure in each tensor mode during training, while \textbf{\method} achieves convergence through mode-wise projections. All the proofs and background details are in Appendix sections~\ref{sec:tensor1}, \ref{sec:tensor2}, and \ref{sec:theory}.

\begin{theorem}[\textbf{\method} Convergence]
For a gradient tensor $\mathcal{G}_t \in \mathbb{R}^{I_1 \times I_2 \times \cdots \times I_d}$, let $\{P_k \in \mathbb{R}^{I_k \times r_k}\}_{k=1}^d$ be fixed orthonormal projection matrices for each mode k with ranks $\{r_k\}_{k=1}^d$.  Suppose for each mode k:
\begin{itemize}
    \item $\mathcal{A}_i$, $\mathcal{B}_i$, $\mathcal{C}_i$ have $L_A^{(k)}$, $L_B^{(k)}$, $L_C^{(k)}$ mode-k continuity, $\|\mathcal{W}_t\|_{(k)} \leq D_k$ (mode-k spectral norm bound), $\hat{\mathcal{B}}_{it}^{(k)} := P_k^\top \mathcal{B}_i^{(k)}(\mathcal{W}_t) P_k$, $\hat{\mathcal{C}}_{it}^{(k)} := P_k^\top \mathcal{C}_i^{(k)}(\mathcal{W}_t) P_k$ ,$\kappa_t^{(k)} := \frac{1}{N}\sum_i \lambda_{\min}(\hat{\mathcal{B}}_{it}^{(k)})\lambda_{\min}(\hat{\mathcal{C}}_{it}^{(k)})$
\end{itemize}

Then \textbf{\method} with $\rho_t \equiv 1$ satisfies for each mode k:
\[
\|(\mathcal{R}_t)_{(k)}\|_F \leq \left[1-\eta(\kappa_{t-1}^{(k)}-L_A^{(k)}-L_B^{(k)}L_C^{(k)}D_k^2)\right] \|(\mathcal{R}_{t-1})_{(k)}\|_F
\]

As a result, if $\min_{t,k} \kappa_t^{(k)} > L_A^{(k)} + L_B^{(k)}L_C^{(k)}D_k^2$ for all modes k, then $\mathcal{R}_t \to 0$ and \textbf{\method} converges with the fixed projections $\{P_k\}_{k=1}^d$. Proof~\ref{proof:theorem1}.

\begin{remark}[Mode-k Continuity]
The mode-k continuity assumption on $\mathcal{A}_i$, $\mathcal{B}_i$, $\mathcal{C}_i$ is mild and holds generically for neural network parameters.
\end{remark}
\end{theorem}

\section{Experimental Setup and Results}
We conduct a comprehensive evaluation of \method\ on a diverse set of benchmark datasets for NOs, representing a range of PDEs with varying complexity and dimensionality.

\begin{figure}
    \centering
    \begin{subfigure}[b]{0.51\textwidth}
         \includegraphics[width=\linewidth]{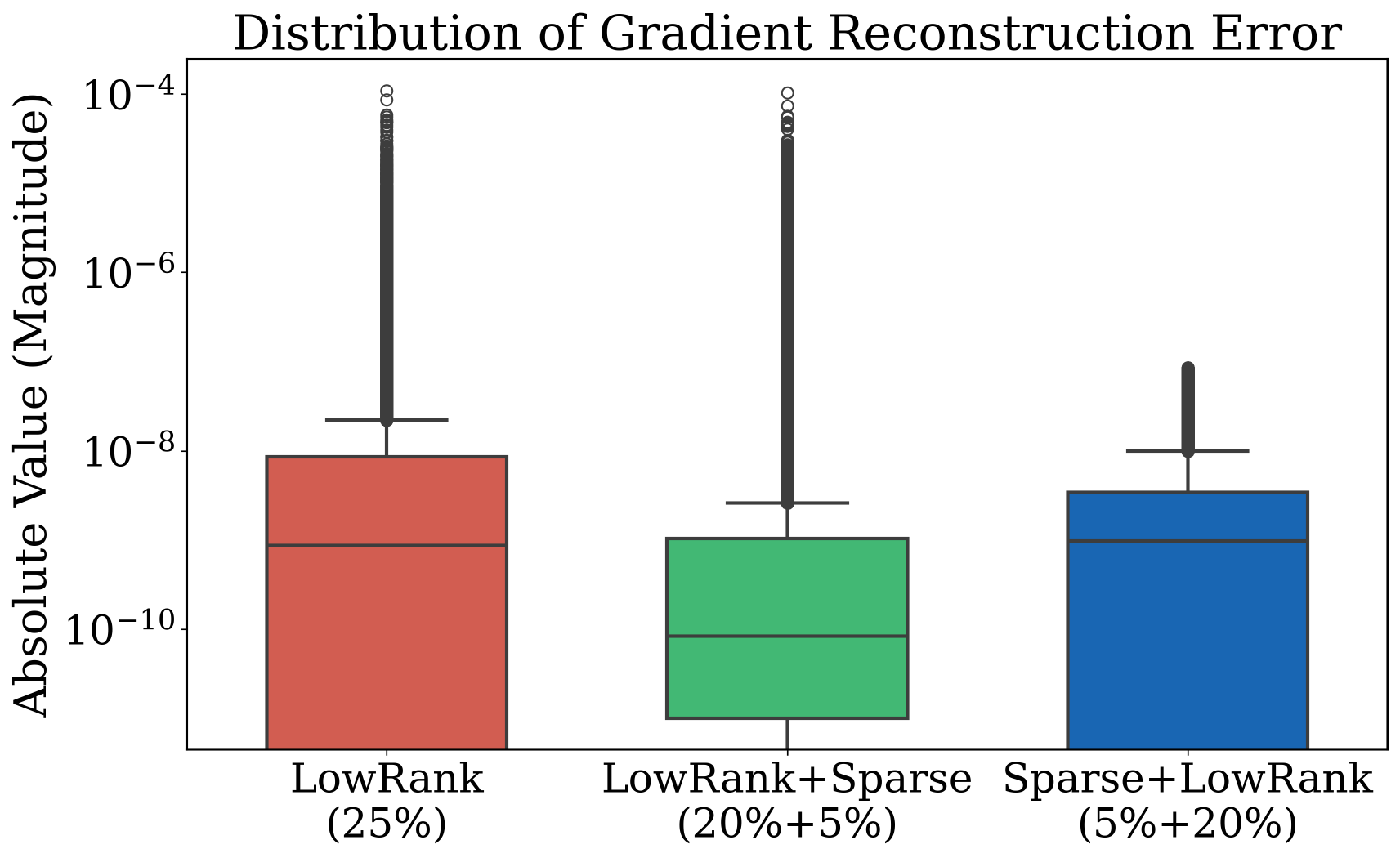}
         \label{fig:outlier_comparison}
    \end{subfigure}
    \hfill
    \begin{subfigure}[b]{0.455\textwidth}
         \includegraphics[width=\linewidth]{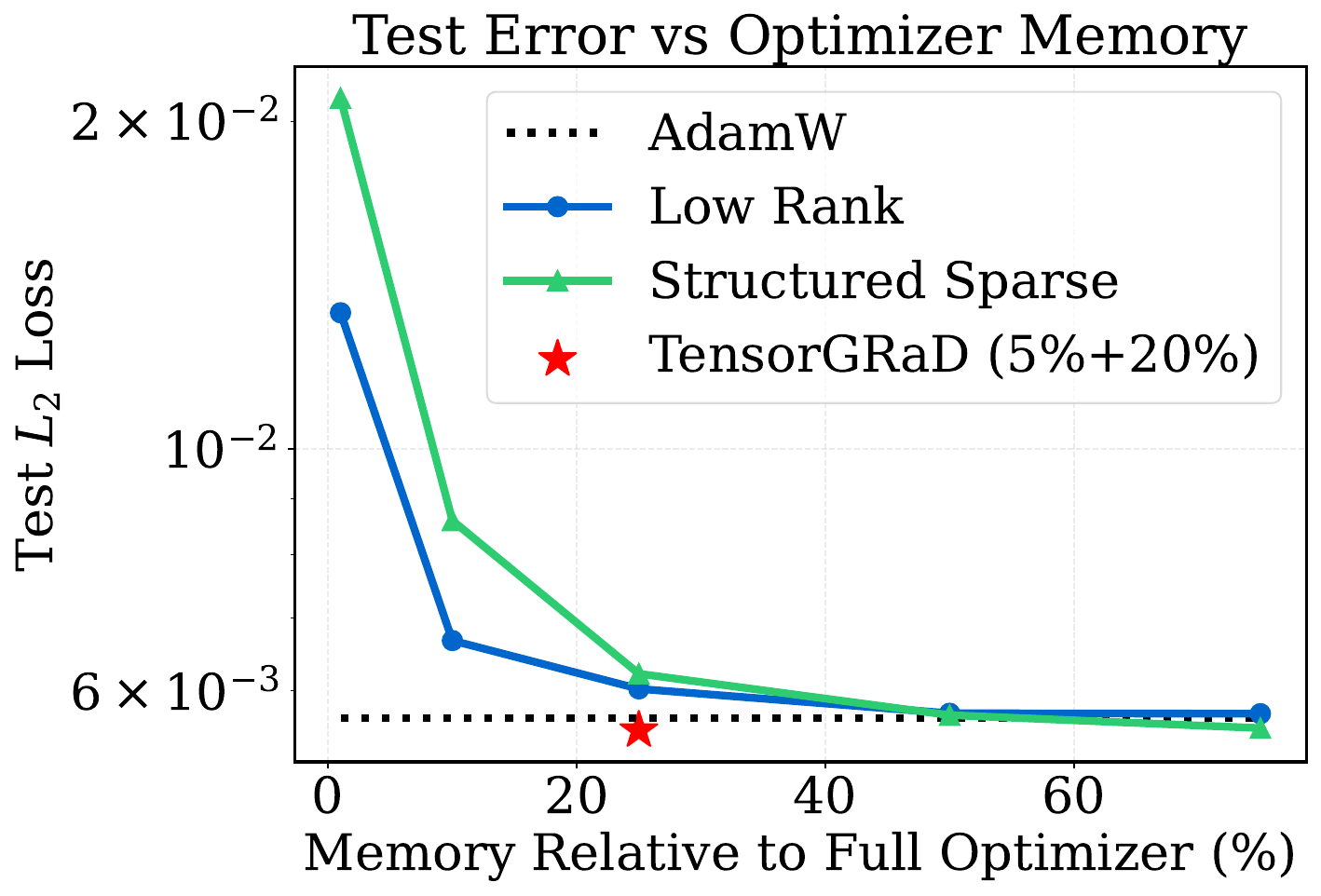}
         \label{fig:mixed_precision_left_memory_righthandside}
    \end{subfigure}\vspace{-10pt}
    \caption{
    \textbf{Left: Gradient reconstruction error.} Box plot of gradient reconstruction error for different compression strategies on a complex FNO layer. \method\ variants reduce high-magnitude outliers. 
    \textbf{Right: Accuracy–memory trade-off.} Performance of low-rank and structured sparse across compression ratio vs.\ \method\ at 25\% and Adam baselines. \method\ achieves the best trade-off at.}
    \label{fig:robust_and_tradeoff}\vspace{-10pt}
\end{figure}

\paragraph{Datasets.}
We report results on the several PDE datasets: 
\textbf{Burgers Equation:} A one-dimensional nonlinear PDE with viscosity modeling fluid dynamics, trained on 1000 samples of Gaussian random fields at 128-point resolution.
\textbf{Darcy Flow:} An elliptic PDE describing fluid flow through porous media with variable coefficients, trained on 4000 samples discretized on a 421×421 grid.
\textbf{Electromagnetic Wave Propagation:} A complex-valued nonlinear Schrödinger equation modeling optical pulse propagation in waveguides with second-harmonic generation, trained on 800 samples with varying physical parameters.
\textbf{Navier-Stokes} We study the 2D Kolmogorov flow, a variant of the incompressible Navier–Stokes equations with periodic forcing~\cite{wang2024beyond}. This dataset is particularly challenging and has a Reynolds number of $Re \approx 2 \times 10^5$, representing a highly turbulent regime.
Full dataset specifications are provided in Appendix~\ref{app:dataset}.

\paragraph{Model Architecture, Training, and Evaluation.}
All models are based on the Fourier Neural Operator (FNO) architecture and trained using the Adam optimizer. Training details, including learning rates, batch sizes, and loss functions, are provided in Appendix~\ref{tab:detailed_fno_arch}. For memory profiling methodology, see Appendix~\ref{app:memory_profiling}. Code will be provided in the supplementary. 

In evaluating \method, we vary the total compression ratio by adjusting the rank of the low-rank decomposition and the density of the sparse tensor to assess the trade-off between memory efficiency and performance.  To further reduce memory usage, we apply \textbf{activation checkpointing}~\cite{chen2016trainingdeepnetssublinear}, which recomputes intermediate activations during backpropagation.
All models are implemented in PyTorch and trained on NVIDIA A100, H100, and H200 GPUs. 
The main paper focuses on the most challenging Navier–Stokes at high resolution dataset and ablations of different sparse–low-rank combinations for \method. We present additional experiments in the appendix, including comparisons to the direct mode-unfolding extension of GaLore. 
We report performance using the $L_2$ test loss, and quantify improvements wrt. baseline performance via relative gain percentage.

\begin{table}[h]
\centering
\caption{
\textbf{Memory and accuracy comparison on Navier–Stokes} $1024\times1024$ with Reynolds number $10^{5}$.  
Train and test losses are $L_{2}\times10^{-2}$ (mean ± 1 standard error over three seeds).  
“Mixed” uses half-precision weights and gradients with a mixed-precision forward pass.  
Memory is a rounded peak GPU allocation.}
\begin{tabular}{lccccc}
\toprule
Model & Rank & Memory (GB) & Precision & Train $L_2$ & Test $L_2$ \\
\midrule
Low-Rank Only & 25\% & 46 & Full  & \num{5.37} $\pm$ \num{0.08} & \num{17.19} $\pm$ \num{0.23} \\
              &      & 29 & Mixed & \num{6.92} $\pm$ \num{0.19} & \num{17.09} $\pm$ \num{0.19} \\
\midrule
Sparse Only   & 25\% & 46 & Full  & \num{6.39} $\pm$ \num{0.32} & \num{18.73} $\pm$ \num{0.08} \\
              &      & 29 & Mixed & \num{7.37} $\pm$ \num{0.14} & \num{18.54} $\pm$ \num{0.34} \\
\midrule
\textbf{\method} & 5\%+20\% & 46 & Full  & \num{5.36} $\pm$ \num{0.05} & \textbf{16.82} $\pm$ \num{0.18} \\
                 &           & 29 & Mixed & \num{6.42} $\pm$ \num{0.15} & \num{16.87} $\pm$ \num{0.15} \\
\midrule
Adam Baseline & 100\% & 52 & Full  & \num{3.94} $\pm$ \num{0.22} & \num{17.02} $\pm$ \num{0.18} \\
               &        & 37 & Mixed & \num{4.86} $\pm$ \num{0.26} & \num{17.01} $\pm$ \num{0.19} \\
\bottomrule
\end{tabular}
\label{tab:chuwei1024data}
\end{table}

\paragraph{\textbf{Results on Navier–Stokes $\mathbf{1024\times1024}$}.}We present results on the Navier–Stokes dataset at $1024\times1024$ resolution with $\mathrm{Re}=10^5$, focusing on the performance–memory trade-offs achieved by \method. Further results for other PDEs and detailed ablations are included in the appendix. We highlight how combining low-rank and sparse gradient compression outperforms each technique in isolation and demonstrate that \method\ is compatible with mixed-precision training. 

Tab.~\ref{tab:chuwei1024data} summarizes the trade-off between memory and accuracy on the challenging NS1024 dataset. \method\ achieves the best performance across all settings: using just a 25\% optimizer state (5\% unstructured top-$k$ sparse entries and 20\% low-rank), it achieves the lowest test loss of $16.82 \times 10^{-2}$ surpassing all baselines, including full-precision Adam ($17.02 \times 10^{-2}$).
Pure 25\% low-rank compression ($17.19 \times 10^{-2}$) and 50\% structured sparsity ($17.21 \times 10^{-2}$) both are close to matching Adam, but structured sparsity at 25\% density degrades to $18.54 \times 10^{-2}$. This suggests that for high-resolution turbulent flows, low-rank methods are more effective than structured sparse updates.

Mixed-precision training further reduces memory without harming accuracy. With \method, mixed-precision yields $16.87 \times 10^{-2}$ and still matches full-precision Adam while reducing the total memory by 55\%. We also observe that methods that match or exceed Adam in test loss often exhibit higher training loss, suggesting that compression introduces a beneficial regularization effect. In contrast, the direct tensorized GaLore approach performs poorly despite using more memory, highlighting the importance of structure-aware gradient compression.

\paragraph{Sparse and low-rank combinations.}
We evaluate different combinations of low-rank (LR) and sparse gradient compression in \method, varying the order, sparsity type, and selection strategy. We distinguish between structured sparsity (SS), which selects aligned slices across modes, and unstructured sparsity (US), which selects arbitrary entries. For selection strategies, we compare top-$k$ (based on magnitude) and rand-$k$ (uniform sampling). In sequential variants (denoted $A \rightarrow B$), the first component receives the full gradient and the second compresses the residual. Additive variants (denoted $A + B$) apply both directly to the gradient and sum their outputs.

The best-performing configuration is $\text{US} \rightarrow \text{LR}$ with $5\%$ sparsity and $20\%$ low-rank, achieving test losses of $5.72$ (top-$k$) and $5.73$ (rand-$k$), outperforming both compression techniques applied in isolation. In contrast, reversing the order ($\text{LR} \rightarrow \text{US}$) leads to higher test losses (e.g., $6.29$ and $6.20$), indicating that removing outliers first improves the quality of the low-rank basis. Structured sparsity performs worse across all variants; for instance, $\text{LR} \rightarrow \text{SS}$ at $5\% + 20\%$ results in test errors $6.72$ (top-$k$) and $6.35$ (rand-$k$). 

Despite some variants achieving similar accuracy, their practicality may differ. For example, configurations with higher unstructured sparsity require storing a large index set, increasing memory and compute overhead. In contrast, the $5\%+20\%$ $\text{US} \rightarrow \text{LR}$ setup balances accuracy and memory.

\begin{table}[h]
\centering
\setlength{\tabcolsep}{5pt}
\caption{
\textbf{Test $L_2$ loss ($\times 10^{-3}$) for different combinations of low-rank (LR), structured sparse (SS), and unstructured sparse (US) gradient updates}. Each cell shows top-$k$ / rand-$k$ results. Sequential forms (denoted $A \rightarrow B$) apply $A$ to the full gradient and $B$ to the residual. “Sum” applies both independently to the full gradient and sums the results.
}
\begin{tabular}{lccc}
\toprule
Method (topk / randk) & 20\%+5\% & 45\%+5\% & 5\%+20\% \\
\midrule
LR $\rightarrow$ SS           & 7.09 / 6.44 & 6.91 / 6.32 & 6.72 / 6.35 \\
LR $\rightarrow$ US           & 6.29 / 6.20 & 6.22 / 6.19 & 6.47 / 6.26 \\
SS $\rightarrow$ LR           & 6.56 / 6.26 & 6.66 / 5.96 & 6.22 / 6.21 \\
US $\rightarrow$ LR           & 6.24 / 6.10 & 6.19 / 6.03 & \textbf{5.72} / 5.73 \\
LR + US (sum)                 & 6.18 / 6.12 & 6.17 / 6.06 & --         \\
\bottomrule
\end{tabular}
\label{tab:lowrank_sparse_ablation_grouped}
\end{table}

\paragraph{\textbf{Mixed-precision training}.}\quad
 We evaluate three different precision configurations on the Navier–Stokes 128 dataset ($128 \times 128$, $\mathrm{Re}=10^3$). In the first setting, all tensors and optimizer states are stored in full precision. The second, referred to as Mixed-1, uses half precision for weights, activations, and gradients except the Fast Fourier Transform (FFT) part (see Sec.~\ref{sec:mixed_precision}), while keeping optimizer states in full precision. Mixed-2 is identical to Mixed-1, except that optimizer states are also stored in half precision. 

Results in Tab.~\ref{tab:ns128_precision_ablation} show that in full precision, \method\ (25\%) matches Adam ($5.72$ vs.\ $5.66 \times 10^{-3}$) and outperforms low-rank (25\%: $6.02$) and structured sparse (25\%: $6.22$). Structured sparsity at 50\% performs best ($5.54$), but \method\ offers better efficiency at lower memory.

When moving to Mixed-1 precision \method\ maintains strong performance, with a test loss of $5.71 \times 10^{-3}$. This matches the full-precision setting and confirms that \method\ remains stable under reduced numerical precision. Notably, some models improve slightly (e.g., Adam drops to $5.62$), consistent with prior observations that mixed precision can act as a mild regularizer~\cite{tu2024guaranteedapproximationboundsmixedprecision}.

In contrast, the Mixed-2 setup, where optimizer states are also stored in half precision, leads to significant degradation. \method\ drops to $7.10$, Adam to $6.92$, and low-rank (25\%) to $7.21$.
We saw that Adam’s moment tensors often contain values on the order of $10^{-5}$ on both the real and imaginary parts, and half precision lacks sufficient mantissa bits to represent them accurately.

These results underscore two key findings. First, \method\ is fully compatible with mixed-precision training, provided that optimizer states are maintained in full precision. Second, \method's low-rank and sparse optimizer states preserve essential gradient information more effectively than directly storing them in reduced precision. This makes \method\ a compelling approach for reducing memory without sacrificing accuracy, especially when paired with mixed precision. Our improved memory is demonstrated in Fig.~\ref{fig:memuse_comparison}.

\begin{table}[h]
\centering
\caption{\textbf{Navier–Stokes ($128 \times 128$, $\mathrm{Re} = 10^3$) precision ablation}: test $L_{2}\times10^{-3}$ under three precision schemes.  
\textbf{FP}: full precision.  
\textbf{Mixed-1} stores gradients, weights, and activations (except FFT) in half precision while keeping optimizer states in full.  
\textbf{Mixed-2} is identical to Mixed-1 but also stores optimizer states in half precision.  
\textbf{LR}: low-rank; \textbf{US/SS}: unstructured/structured sparse.  
Lower is better; best per column is \textbf{bold}.}
\label{tab:ns128_precision_ablation}
\setlength{\tabcolsep}{6pt}  %
\begin{tabular}{lccc}
\toprule
Method & Full & Mixed-1 & Mixed-2 (Half Optim. states) \\
\midrule
Adam          &\textbf{ 5.66} & 5.62 & \textbf{6.92} \\
LR 50\%       & 5.71 & 5.70 & 7.08 \\
LR 25\%       & 6.02 & 5.87 & 7.21 \\
SS 50\%       & 5.54 & \textbf{5.56} & 7.14 \\
SS 25\%       & 6.22 & 6.14 & 7.78 \\
LR+US 25\% (5+20) & 5.72 & 5.71  & 7.10 \\
\bottomrule
\end{tabular}
\end{table}

We show additional results in the Appendix~\ref{app:additional_results}. In Appendix~\ref{sec_app:update_freq}, we ablate the effect of different update frequencies for computing factor matrices for the tensor low-rank part and for the sparse part. Appendix~\ref{sec_app:structured_sparsity} compares structured sparsity patterns using top-$k$ versus random-$k$ selection strategies. In Appendix~\ref{sec_app:matricizing_vs_tensor}, we provide a detailed comparison between our proposed tensor-based low-rank decomposition and a baseline method that applies a GaLore-style low-rank projection to matricized tensors. Finally, Appendix~\ref{tab:1} includes extended benchmark results across multiple datasets: Burgers, Darcy, ElectroMagnetic, and Navier–Stokes 1024.

\section{Related Work}
\label{sec:2}
Our work, \textbf{\method}, introduces a novel approach to efficiently training neural operators by decomposing gradients. While significant work has been done in related areas, the specific approach of gradient decomposition in tensors has not been explored. \textbf{Tensor Methods in Deep Learning:} Tensor decomposition has been widely used to compress deep networks~\cite{novikov2015tensorizing, lebedev2015speeding, kim2016compression,9420085,PANAGAKIS20241009}, but these methods focus on weight tensors rather than gradients during training. \textbf{Sparse gradient updates:}
GRASS~\cite{2024grass} introduced structured sparsity for matrices using sampling strategies like Top$-k$ magnitude sampling. In tensor settings, this approach requires unfolding tensors, disrupting the inherent mode-wise structure. Instead, we use unstructured sparsity, selecting individual tensor entries directly without unfolding them.

\textbf{Neural Operators:} Recent advancements have led to neural operators~\cite{li2020fourier, kovachki2021neural}, with FNOs showing remarkable success in scientific computing tasks~\cite{li2021fourier}. However, these methods have not explored gradient decomposition for memory efficiency.
\textbf{Efficient Training Techniques:} Various approaches reduce memory footprints of large models. LoRA~\cite{hu2022lora} adds a low-rank weight matrix to a frozen pre-trained matrix. FLoRA~\cite{si2024floralowrankcorespace} extends this to higher dimensions using Tucker decomposition. For neural operators, MG-TFNO~\cite{kossaifi2024multi} combines tensor decomposition with multi-grid approaches, while iFNO~\cite{george2024incrementalspatialspectrallearning} incrementally scales FNO weight ranks during training.
\textbf{Low-rank Plus Sparse:}
Robust PCA~\cite{candes2009robustprincipalcomponentanalysis} separates low-rank matrices from sparse noise, extended to tensors by RTD~\cite{NIPS2014_robust_tensor_decomposition}. Hybrid decompositions have been explored for model weights~\cite{sltrain2024} and attention matrices~\cite{chen2021scatterbrain}, but not for compressing gradient tensors.
\textbf{Mixed Precision Training} utilizes lower precision formats for certain operations, reducing memory usage and potentially accelerating training on compatible hardware~\cite{tu2024guaranteedapproximationboundsmixedprecision}.
\textbf{Combination with existing methods} \method\ can complement many existing techniques, potentially leading to greater memory benefits by integrating with methods like FLoRA or MG-TFNO and frameworks like iFNO.

\section{Conclusion}
We presented \textbf{\method}, a memory-efficient optimization framework for training large-scale tensor-structured models.
By combining low-rank factorization with unstructured sparse updates in a robust tensor decomposition of the gradients, \method\ achieves substantial memory savings without sacrificing performance.
We further introduce a mixed-precision training strategy that complements our method and improves efficiency.
 We validate our findings on challenging PDE benchmarks and thorough ablations. Our approach enables training of large-scale neural operators on high-resolution PDE data.

\paragraph{Limitations.}\quad 
While \method\ provides substantial memory savings and strong empirical performance, it also has limitations. First, the computational overhead of Tucker decomposition may be nontrivial, despite amortization via infrequent updates. Second, selecting appropriate ranks and sparsity levels remains a challenge, particularly for heterogeneous datasets where optimal compression may vary across layers or training stages. Third, although we show \method\ is compatible with mixed precision training, it currently assumes full-precision optimizer states; future extensions could explore quantized or compressed representations. Finally, our current evaluation is limited to PDE-driven neural operator: applying \method\ to other domains such as vision or language models with tensorized architectures is future work.

\paragraph{Broader Impact.}\quad By enabling high-resolution scientific models to train on commodity hardware, \method\ lowers the barrier to entry for resource-constrained researchers and institutions. This may democratize access to cutting-edge simulation tools for climate modeling, fluid dynamics, etc. At the same time, the efficient training enabled by \method\ can reduce the energy footprint of large-scale scientific machine learning workloads. We are excited to see \method\ adopted in broader scientific applications.

\section{Acknowledgements.}
Sebastian Loeschcke is supported by the Danish Data Science Academy, which is funded by the Novo Nordisk Foundation (NNF21SA0069429) and VILLUM FONDEN (40516). Robert Joseph George is supported by a Caltech Graduate Fellowship. David Pitt is supported by the Schmidt Scholars in Software Engineering program. Anima Anandkumar is supported by the Bren Named Chair, Schmidt AI 2050 Senior fellow, and ONR (MURI grant N00014-18-12624).

\newpage
\appendix
\section*{Appendix}

\section{Dataset}
\label{app:dataset}
\subsection{Navier-Stokes Datasets}

\paragraph{Navier-Stokes 1024:} 

\begin{figure}
    \centering
    \includegraphics[width=\textwidth]{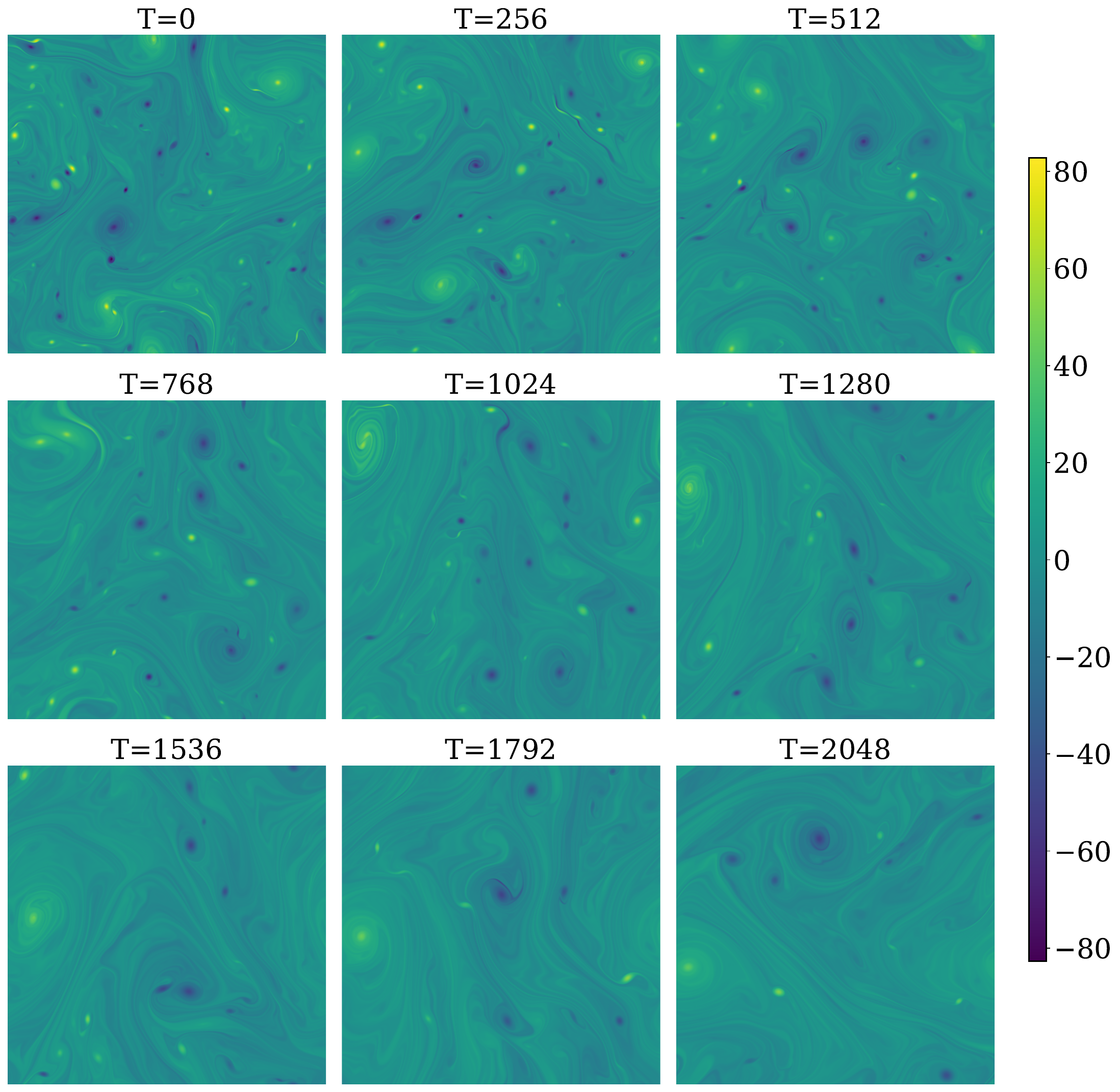}
    \caption{Navier-Stokes $1024\times 1024$ with Reynolds number at $2\times 10^5$}
    \label{fig:Navier_stokes_1024}
\end{figure}

We use the 2D Kolmogorov flow of Wang et al.~\cite{wang2024beyond}, a periodically forced, incompressible variant of the Navier–Stokes equations. The velocity field $\mathbf u(x,y,t) \in \mathbb{R}^2$ evolves on a periodic domain $[0, 2\pi]^2$ according to:
$$
\partial_t \mathbf{u} = - (\mathbf{u} \cdot \nabla) \mathbf{u} - \nabla p + \nu \Delta \mathbf{u} + (\sin(4y), 0)^T, \quad \nabla \cdot \mathbf{u} = 0, \quad (x,y,t) \in [0, L]^2 \times \mathbb{R}_+,
$$
Our analysis focuses on the vorticity form, where the vorticity $\omega = \nabla \times \mathbf{u}$ evolves as:
$$
\partial_t \omega = - \mathbf{u} \cdot \nabla \omega + \nu \Delta \omega + \nabla \times (\sin(4y), 0)^T.
$$
The behavior of this flow is characterized by the Reynolds number $Re = \frac{\overline{u} l}{\nu}$, where $\nu$ is the kinematic viscosity, $\overline{u}$ is the root-mean-square velocity, and $l$ is the characteristic length. Higher $Re$ values correspond to more turbulent flows. In our setup, $L = 2\pi$ and $\nu = 10^{-4}$, leading to $Re \approx 2 \times 10^5$, representing a highly turbulent regime.

\paragraph{Initial Condition and Data Collection:}
The initial velocity field $\mathbf{u_0}(x)$ is sampled from a Gaussian random field $\mathcal{N}(0, C)$ with covariance 
\[
C = 7^{3/2} (-\Delta + 49I)^{-2.5}.
\]
Data is collected from $t = 30$ onward, ensuring a statistically steady state. Our dataset, sourced from~\cite{wang2024beyond}, comprises $[7681, 1024, 1024]$ entries, from which we generate 1577 training samples and 279 test samples, with $\Delta t=256$. 
Figure~\ref{fig:Navier_stokes_1024} illustrates 9 representative samples from the dataset, selected at uniform intervals of 256 steps.

\paragraph{Input-Output Pair Construction:}
Input-output pairs are defined by:
\begin{itemize}
    \item $T$: The timestep difference between the input and output.
    \item $t_{\text{skip}}$: The number of frames skipped between samples.
\end{itemize}

For example, with indices $[0, 1, 2, \ldots, 32]$, setting $T=16$ and $t_{\text{skip}}=8$ creates pairs $[(0,16), (8,24), (16,32)]$. Our setup uses $T = 256$ and $t_{\text{skip}} = 4$, resulting in 1577 training samples and 279 test samples.

\paragraph{Navier-Stokes (NS128):} We also experiment with the two-dimensional Navier-Stokes equation in vorticity form:
\begin{equation}
\begin{aligned}
\partial_t \omega + \nabla^\perp \phi \cdot \omega &= \frac{1}{Re} \Delta\omega + f, \quad x \in \mathbb{T}^2, t \in (0, T] \\
-\Delta\phi &= \omega, \quad \int_{\mathbb{T}^2} \phi = 0, \quad x \in \mathbb{T}^2, t \in (0, T]
\end{aligned}
\end{equation}
with a Reynolds number $Re = 10^3$ and final time $T = 5$. The domain is discretized on a $1024 \times 1024$ grid. We generate 10,000 training samples and 2,000 test samples using a pseudo-spectral method. To test scalability, we also evaluate our approach on a subsampled resolution of $128 \times 128$. Memory profiling is performed at the full $1024 \times 1024$ resolution.

\paragraph{Burgers Equation:} We consider the one-dimensional Burgers equation on the torus:

\begin{equation}
\partial_t u + u u_x = \nu u_{xx}, \quad x \in \mathbb{T}, t \in (0, T]
\end{equation}

with initial condition $u_0 \in L^2(\mathbb{T}; \mathbb{C})$ and viscosity $\nu > 0$. We set $T = 1$ and $\nu = 0.01$. Input functions are sampled from a Gaussian random field, and solutions are obtained using a pseudo-spectral method. We use 1000 samples for training and 200 for testing, with 128 resolution.

\paragraph{Darcy Flow:} The Darcy flow problem is defined by the elliptic PDE:

\begin{equation}
-\nabla \cdot (a(x) \nabla u(x)) = f(x), \quad x \in (0,1)^2
\end{equation}

with boundary conditions $u(x) = 0$ for $x \in \partial(0,1)^2$. The input $a$ is sampled from a Gaussian random field, and $f$ is fixed. We use 4000 training samples and 1000 test samples, with the domain discretized on a 421 × 421 grid.

\paragraph{Electromagnetic Wave Propagation:}
Lastly, we present a dataset that represents complex-valued data inherently. We consider the propagation of optical pulses in a nonlinear waveguide with second-order nonlinearity ($\kappa^2$). The problem is governed by the nonlinear Schrödinger equation (NLSE) with additional terms for second-harmonic generation:
\begin{equation}
\frac{\partial A}{\partial z} = -\mathrm{i}\frac{\beta_2}{2}\frac{\partial^2 A}{\partial t^2} + \mathrm{i}\gamma|A|^2A + \mathrm{i}\kappa A^*e^{\mathrm{i}\Delta kz}
\end{equation}

where $A$ is the complex electric field envelope, $\mathrm{i}$ is the imaginary unit, $z$ is the propagation distance, $t$ is time, $\beta_2$ is the group velocity dispersion, $\gamma$ is the nonlinear parameter, $\kappa$ is the coupling coefficient for second-harmonic generation, and $\Delta k$ is the phase mismatch. Our dataset consists of 800 training samples and 200 testing samples. 
The input consists of several parameters: the poling region length ranging from 2mm to 15mm, the poling period mismatch varying from -50nm to +50nm, and the pump pulse energy spanning from a few fJ to thousands of fJ. Additionally, the input includes the complex electric field envelope of the input pulse. The output of the system is the complex electric field envelope of the resulting output pulse.

\section{FNO Memory Usage}

Fig.~\ref{fig:profiling} illustrates the memory usage patterns in Fourier Neural Operators (FNOs) as the number of modes increases. This analysis provides crucial insights into the scalability challenges faced when training large FNO models.

\begin{figure}[h]
\centering
\includegraphics[width=0.8\textwidth]{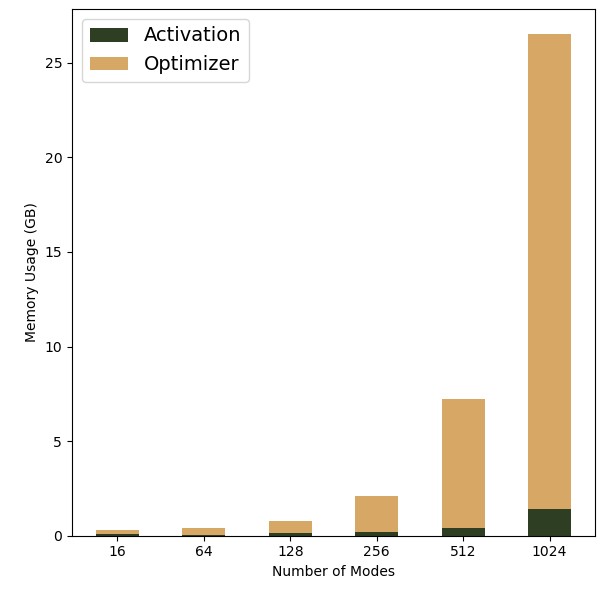}
\caption{Memory usage in FNO as a function of the number of modes}
\label{fig:profiling}
\end{figure}

\begin{figure}
    \centering
    \includegraphics[width=1.0\linewidth]{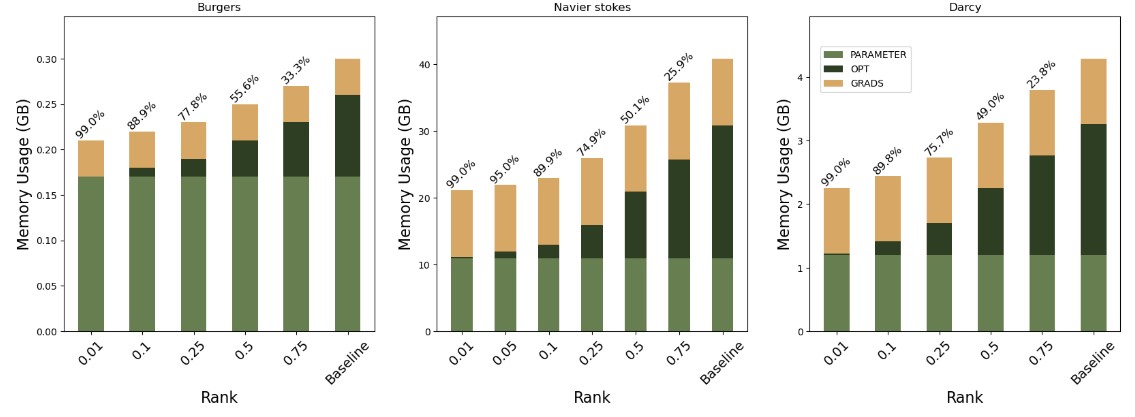}
    \caption{Memory usage of FNO and GINO on various datasets on an NVIDIA A100. On top of the bars, we showcase the reduction in optimizer memory in \% using \textbf{\method}.}
    \label{fig:profiling2}
\end{figure}

As evident from the figure, the memory consumption is divided into two main categories: activation memory and optimizer memory. The activation memory, represented by the dark green bars, remains relatively constant and low across different numbers of modes. This stability in activation memory is a positive attribute of FNOs, indicating that the forward and backward passes do not significantly increase memory requirements as the model complexity grows.

However, the optimizer memory, shown in yellow, exhibits a dramatic increase as the number of modes grows. This exponential growth in optimizer memory becomes particularly pronounced for models with more than 128 modes. For instance, when the number of modes reaches 1024, the optimizer memory dominates the total memory usage, far exceeding the memory required for activations.

This trend highlights a critical bottleneck in scaling FNO models to higher resolutions or more complex problems. The optimizer's memory footprint, which includes storage for gradients, momentum, and adaptive learning rate parameters, becomes the primary limiting factor. This observation motivates the need for memory-efficient optimization techniques like \textbf{\method}, which specifically target the reduction of optimizer memory usage while maintaining model performance.

\section{Matrix GaLore}

GaLore (Gradient Low-Rank Projection)~\cite{zhao2024galore} is a memory-efficient optimization method designed to reduce the memory overhead of gradient updates by leveraging the low-rank structure often present in gradient matrices. 
Specifically, GaLore operates on weight matrices $W \in \mathbb{R}^{m \times n}$ and their corresponding gradient matrices $G \in \mathbb{R}^{m \times n}$. For a given rank $r$, GaLore computes the SVD of the gradient matrix every $T$ steps, forms projection matrices using the first $r$ singular vectors, then projects the gradient onto this low-rank subspace to perform optimization. After computing the optimizer update, the gradients are projected back to their full rank for use in the model. This approach allows GaLore to maintain a low memory footprint by storing and updating only the low-rank representations of gradients.

\paragraph{Challenges of applying GaLore to neural operators.}\quad
\label{sec:naive_galore}
In order to apply standard GaLore to tensor weights, the weights must first be reshaped into a matrix to compute the SVD for projection into a low-rank space. GaLore takes one rank parameter, $r$, and projects high-rank gradients onto the first $r$ basis vectors of the corresponding SVD rotation matrix. When the weight matrix corresponds to an operator that maps between vectors, a single rank cutoff can be applied while preserving most information. 

However, in the tensor case, weights correspond to higher-order maps between function spaces. Depending on the chosen strategy for reshaping tensor weights into a matrix, applying a single-dimension rank cutoff to the matrix may discard key information - for instance, for a tensor $W \in \mathbb{C}^{A \times B \times m \times m}$, where $A$ is the number of input channels, $B$ is the number of output channels, and $m$ is the number of truncated Fourier basis modes along each dimension, reshaping $W$ into $W' \in \mathbb{C}^{ABm \times m}$ and cutting off the first dimension at rank $r$ may remove all information about Fourier modes along the first dimension, making function learning impossible. We call this method \textit{GaLore} and provide several comparisons to demonstrate its flaws. 

One flaw is the \textbf{Loss of mode-specific information}: by collapsing multiple tensor dimensions into one matrix dimension, we lose the ability to preserve different amounts of information along each tensor mode. The other is that we have an \textbf{imbalanced projection}: Projecting only on one side of the reshaped matrix (e.g., only $U$ or only $V$ from the SVD) can severely limit the operator's capacity. However, projecting on both sides often leads to training instability and failure to converge. This method also encounters \textbf{rank selection issues}: Choosing a single rank cutoff for the reshaped matrix makes it difficult to balance information preservation across all the original tensor dimensions. A rank that preserves enough information for one dimension may be too restrictive for another.

\section{Tensor Decomposition}

Tensors are multidimensional arrays that generalize the concepts of vectors (first-order tensors) and matrices (second-order tensors) to higher orders. An $N$th-order tensor $\mathcal{X} \in \mathbb{C}^{I_1 \times I_2 \times \cdots \times I_N}$ is an $N$-way array where each mode $n$ has dimension $I_n$. Like matrices, in tensors, we can decompose the tensors into low-rank factors using the Tucker decomposition, also known as the higher-order SVD (HOSVD), which decomposes a tensor into a core tensor multiplied by a matrix along each mode:

\begin{equation}
    \mathcal{X} \approx \mathcal{G} \times_1 U^{(1)} \times_2 U^{(2)} \cdots \times_N U^{(N)} = \llbracket \mathcal{G}; U^{(1)}, U^{(2)}, \ldots, U^{(N)} \rrbracket
\end{equation}

where $\mathcal{G} \in \mathbb{C}^{R_1 \times R_2 \times \cdots \times R_N}$ is the core tensor, $U^{(n)} \in \mathbb{C}^{I_n \times R_n}$ are factor matrices, and $\times_n$ denotes the $n$-mode product.  Two critical aspects of the Tucker decomposition make it particularly suitable for our \textbf{\method} method:

\begin{enumerate}[wide]
    \item \textbf{Equivalence to SVD in 2D}: In the special case of 2D tensors (matrices), the Tucker decomposition reduces to the familiar SVD. The core tensor $\mathcal{G}$ becomes equivalent to the diagonal matrix $\Sigma$ in SVD, while the factor matrices correspond to the orthogonal matrices $U$ and $V$~\cite{siam}. This property ensures that our method seamlessly extends the principles of matrix-based techniques to higher-order tensors.
\item \textbf{Orthogonality of factor matrices}: The factor matrices $U^{(n)}$ in Tucker decomposition are orthogonal, mirroring the properties of $U$ and $V$ in SVD. This orthogonality is crucial for the efficiency and stability of the GaLore method. Specifically:
\begin{enumerate}[wide]
    \item \textit{Projection efficiency}: The orthogonality allows us to project tensors onto the subspace spanned by these matrices through simple matrix multiplication, without the need for costly inverse computations.
\item \textit{Easy inversion}: When we need to reverse the projection, we can simply use the transpose of these orthogonal matrices instead of computing their inverses. This property is expressed mathematically as $(U^{(n)})^T U^{(n)} = I$, where $I$ is the identity matrix.
\item \textit{Numerical stability}: Orthogonal matrices have a condition number of 1, ensuring that the projection and its inverse are numerically stable operations, even for high-dimensional tensors.
\end{enumerate}

\subsection{Implementation}
\label{sec:impl_hoi}
The \textbf{low-rank component} of \method\ is implemented using TensorLy~\cite{kossaifi2019tensorly}, which provides an efficient implementation of the Tucker decomposition based on Higher-Order Orthogonal Iteration (HOI). Given an input tensor $\mathcal{X}$, HOI approximates the factor matrices ${U^{(n)}}_n$ by computing truncated SVDs of mode-$n$ unfoldings and iteratively refines them to minimize the Frobenius reconstruction error. The decomposition supports warm restarts, allowing factors to be reused across steps to reduce iteration cost.
All subsequent operations, like compressing the gradients and reconstructing the low-rank updates, are performed using PyTorch.

\end{enumerate}

\section{Structured Sparse Projections}
\label{appendix:sparse}
Sparse projections provide an alternative to low-rank compression by selectively retaining a subset of gradient values. GRASS~\cite{2024grass} introduced this approach for gradient matrices, where optimizer statistics are maintained only for a selected subset of rows, leading to substantial memory savings without requiring an SVD.

This idea can be directly extended to tensor gradients by applying structured sparsity, where the mask $\Omega$ selects entire slices along one or more tensor modes. For a gradient tensor $\mathcal{G} \in \mathbb{C}^{I_1 \times \cdots \times I_N}$, the structured mask is defined as a Cartesian product of per-mode index sets:
$$\Omega = M^{(1)} \times M^{(2)} \times \cdots \times M^{(N)},$$
and the projected tensor is $\hat{\mathcal{G}} = \mathcal{G}[\Omega]$. This form of sparsity is efficient, as the mask stores only $\sum_n r_n$ indices, and projection is implemented via multi-dimensional slicing.

The selection strategies presented in GRASS~\cite{2024grass} can also be applied to tensors for constructing the sets $M^{(n)}$, including random sampling (Rand-$k$), and norm-based rules such as Top-$k$ and Prob-$k$. However, norm-based selection requires unfolding the tensor along each mode and computing slice norms $s^{(n)}i = \lVert \mathcal{G}{i::\cdots} \rVert_2$, which may discard important structural information depending on the unfolding.

Structured sparsity is particularly effective when gradient mass is concentrated along specific tensor modes. However, it lacks flexibility in targeting scattered or irregularly positioned high-importance entries.

To restore the sparse tensor to its original shape during optimizer updates, we use a back-projection operation that places the retained values back into their original positions and fills all other entries with zeros. This is implemented using the adjoint of the projection operator:
$$\tilde{\mathcal{G}} = \mathcal{P}\Omega^{\top}(\hat{\mathbf{g}}) =, \operatorname{unvec}\left(P\Omega^{\top} \hat{\mathbf{g}}\right),$$
where $\hat{\mathbf{g}} = P_\Omega \operatorname{vec}(\mathcal{G})$ is the compressed vector of retained entries.

\section{Profiling Methodology}
\label{app:memory_profiling}

To analyze the performance and memory usage of our \textbf{\method} method, we implemented a comprehensive profiling setup using PyTorch's built-in profiler. This allowed us to gain detailed insights into the computational and memory requirements of our algorithm compared to baseline methods.

\textbf{Detailed Memory Breakdown.} We implemented a detailed memory tracking system to break GPU memory usage down into key categories.

PyTorch provides access to a memory profiler, which collects granular information about each block of memory allocated on the CPU and GPU within the context window in which it is invoked. The profiler is run over a set number of iterations, each of which corresponds to a forward and backward pass through the model and one step of the optimizer. The profiler discards data from the first iteration due to the additional overhead of initialization and allocation, as well as a specified number of warmup steps. The profiler collects data for a set number of iterations, and averages this process over a series of repeats. IN our experiments, we discarded 1 step, warmed up for 1 step, collected data for 3 steps, and performed 3 repeats.

To break down memory by category, we relied on the automatic categorization utility provided by the profiler's Memory Timeline, which coalesces all individual blocks into one of eight categories, which we enumerate and explain below.

\begin{itemize}[leftmargin=*]
    \item \textbf{Model parameters}: Tensors registered as instances of \texttt{torch.nn.Parameter}.
    \item \textbf{Optimizer states:}  Optimizer states are tensors stored within the optimizer that are used to compute the final gradient. For Adam optimizer, this includes first and second moment estimates, which are each a tensor of the same size as the gradients themselves.
    \item \textbf{Inputs:}  memory allocations that occur during data loading and preprocessing.
    \item \textbf{Activations:}  Temporary tensors created during the forward pass of the model. 
    \item \textbf{Gradients:} The tensors added to model parameters during the optimizer step. They are computed for each parameter by backpropagating loss through the model.
    \item \textbf{Autograd Detail:} Extra memory allocated during PyTorch's autograd operation, including memory used for storing computational graphs and intermediate results needed for backpropagation.
    \item \textbf{Temporary:}  Temporary Buffers are short-lived tensors that are created and destroyed within a single operation or a small set of operations. For \textbf{\method}, it is often used in complex computations like FFTs or tensor decompositions within galore.
    \item \textbf{"None":} PyTorch's profiler is often unable to categorize memory immediately upon allocation. In our tests, we compared figures generated by coalescing a profiler's memory timeline one-to-one with more granular block-level traces acquired from PyTorch's CUDA memory snapshot to conclude that memory tagged as None is either eventually reclassified as another type, or corresponds to an intermediate activation that is deallocated during backpropagation. For this reason, we chose to classify "None" memory "Intermediate". For simplicity in the final figure, we also included temporary buffers, autograd detail buffers and activations in the Intermediate category, though our profiler maintains the capability to collect data for each category separately. 
    
\end{itemize}
To accurately break down the profiler's memory timeline into these categories, we took the following approach. First, we noticed that the peak memory allocation often occured at a step when the total of None memory was very high, and that totals reported during that timestep for other categories were zero or very low. To obtain informative breakdowns for those categories, informed by our analysis of memory tagged as None, we identified the timestep at which peak memory usage occured \emph{after excluding None-tagged memory from the total}. This allowed us to break down all non-intermediate memory by its true classification once properly tagged. Once we obtained this breakdown, we subtracted the total from the true total memory allocation observed at the peak to obtain the true amount of intermediate memory allocated at peak.

\begin{algorithm}[!h]
\caption{GaLore}
\begin{algorithmic}[1]
\Require A layer weight tensor $\mathcal{W} \in \mathbb{C}^{N_1 \times N_2 \times N_3 \times N_4}$. Step size $\eta$, scale factor $\alpha$, decay rates $\beta_1, \beta_2$, rank $r$, subspace change frequency $T$, chosen dimension $d$.
\State Initialize first-order moment $\mathcal{M}_0 \in \mathbb{C}^{r \times N_2 \times N_3 \times N_4} \gets 0 $ (Assuming matrization 1)
\State Initialize second-order moment $\mathcal{V}_0 \in \mathbb{R}^{r \times N_2 \times N_3 \times N_4} \gets 0 $ (Assuming matrization 1)
\State Initialize step $t \gets 0$
\Repeat
    \State $\mathcal{G}_t \in \mathbb{C}^{N_1 \times N_2 \times N_3 \times N_4} \gets -\nabla_\mathcal{W} \phi_t(\mathcal{W}_t)$
    \State $G_t^{(d)} \gets \text{Reshape}(\mathcal{G}_t, (N_d, \prod_{i \neq d} N_i))$ \Comment{Reshape tensor to matrix}
    \If{$t \bmod T = 0$}
        \State $U, \Sigma, V^\top \gets \text{SVD}(G_t^{(d)})$ \Comment{Compute SVD}
        \State $P \gets V[:,:r]^\top$ \Comment{Select $r$ right singular vectors}
    \EndIf
    \State $R_t \gets G_t^{(d)} P^\top$ \Comment{Project gradient into compact space}
    \State \textbf{UPDATE}($R_t$) by Adam:
    \State \quad $M_t \gets \beta_1 \cdot M_{t-1} + (1 - \beta_1) \cdot R_t$
    \State \quad $V_t \gets \beta_2 \cdot V_{t-1} + (1 - \beta_2) \cdot |R_t|^2$
    \State \quad $M_t \gets M_t / (1 - \beta_1^t)$
    \State \quad $V_t \gets V_t / (1 - \beta_2^t)$
    \State \quad $N_t \gets M_t / (\sqrt{V_t} + \epsilon)$
    \State $\tilde{G}_t^{(d)} \gets \alpha \cdot N_t P$ \Comment{Project back to original space}
    \State $\tilde{\mathcal{G}}_t \gets \text{Reshape}(\tilde{G}_t^{(d)}, (N_1, N_2, N_3, N_4))$ \Comment{Reshape back to tensor}
    \State $\mathcal{W}_t \gets \mathcal{W}_{t-1} + \eta \cdot \tilde{\mathcal{G}}_t$
    \State $t \gets t + 1$
\Until{convergence criteria met}
\State \Return $\mathcal{W}_t$
\end{algorithmic}
\end{algorithm}

\section{Additional Results}
\label{app:additional_results}

Our experiments demonstrate the effectiveness of \textbf{\method} across various datasets, showing significant improvements in both performance and memory efficiency as shown in Tab.~\ref{tab:1}. 
For the Burgers equation, our method consistently outperformed the baseline FNO, with performance improving as rank increased. On the Darcy flow problem,  \textbf{\method} achieved up to a 50\% gain in test loss at rank 0.25, while reducing optimizer memory by 76\%. The Navier-Stokes experiments showcased \textbf{\method}'s ability to handle complex problems, maintaining comparable performance at lower ranks while dramatically reducing memory usage. Electromagnetic wave propagation simulations saw up to 11\% gains.

\begin{table}[h]
\centering
\caption{
Full results for memory and accuracy comparison on Navier–Stokes $1024\times1024$ with Reynolds number $10^{5}$.  
    Train and test losses are $L_{2}\times10^{-2}$ (mean ± 1 standard error over three seeds).  
    “Mixed” uses half-precision weights and gradients with a mixed-precision forward pass. Memory is a rounded peak GPU allocation.}
\begin{tabular}{lccccc}
\hline
Model & Memory (GB) & Precision & Train  $L_2$ & Test $L_2$   \\
\hline
\hline
\textbf{Low-Rank Only} \\
25\%           & 46 & Full  & \num{5.37} $\pm$ \num{0.08} & \num{17.19} $\pm$ \num{0.23} \\
               & 29  & Mixed & \num{6.92} $\pm$\num{0.19}                   & \num{17.09} $\pm$\num{0.19}                  \\
25\% rand-svd  & 46 & Full  & \num{5.93}  $\pm$\num{0.36}                 & \num{17.23}  $\pm$\num{0.21}                   \\
50\% Matrix (d=1) &  49 & Full  & \num{29.63} $\pm$  \num{1.46}& \num{33.11} $\pm$ \num{1.21}   \\
\hline
\textbf{Sparse Only} \\
25\%           & 46 & Full  & \num{6.39}  $\pm$ \num{0.32}                   & \num{18.73}  $\pm$ \num{0.08}                 \\
               & 29   & Mixed &  \num{7.37}  $\pm$ \num{0.14}                             &  \num{18.54}  $\pm$ \num{0.34}                              \\
50\%           & 49  & Full  & \num{5.35} $\pm$ \num{0.08} & \num{17.21} $\pm$ \num{0.16}\\
               & 32   & Mixed & \num{6.02} $\pm$ \num{0.28}                          & \num{17.38} $\pm$  \num{0.31}                            \\
\hline
\textbf{\method} \\
25\% (5+20)    & 46 & Full  & \num{5.36} $\pm$ \num{0.05} & \num{16.82} $\pm$ \num{0.18} \\
               & 29   & Mixed & \num{6.42} $\pm$ \num{0.15} & \num{16.87} $\pm$ \num{0.15}  \\
\hline
\textbf{Adam Baseline} \\
No Compression & 52 & Full  & \num{3.94} $\pm$ \num{0.22} & \num{17.02} $\pm$ \num{0.18}   \\
               & 37   & Mixed & \num{4.86} $\pm$ \num{0.26} & \num{17.01} $\pm$ \num{0.19}  \\
\hline
\label{tab:chuwei1024data_app}
\end{tabular}
\end{table}

\subsection{Update frequency of tensor decomposition ablation.}
\label{sec_app:update_freq}

Fig.~\ref{fig:gap_ablation_plot} illustrates how the projection gap affects \method\ on the NS128 task over the first 150 epochs, comparing 25\% and 50\% rank settings. In line with findings from~\cite{zhao2024galore,2024grass}, increasing the projection gap can improve generalization, as switching subspaces introduces a small amount of noise that can harm the performance, if done too frequently

We observe that low-rank projections are less sensitive to the projection gap and tend to perform better in the early stages of training. In our experiments, they consistently lead during roughly the first half of training, after which structured sparsity gradually catches up. A likely explanation is that low-rank projections quickly capture the dominant modes of the gradient with fewer parameters, enabling faster initial convergence. In contrast, structured sparsity requires more iterations to recover the same signal, as it discards parts of the gradient during each step. This slower adaptation may, however, contribute to better generalization by reducing overfitting in later epochs.

\begin{figure}
    \centering
    \begin{subfigure}[b]{0.48\textwidth}
         \includegraphics[width=\linewidth] {new_figures/ranks/no_mixed/model_size_comparison_precision.pdf}
    \end{subfigure}
    \hfill
    \begin{subfigure}[b]{0.48\textwidth}
        \includegraphics[width=\linewidth]
        {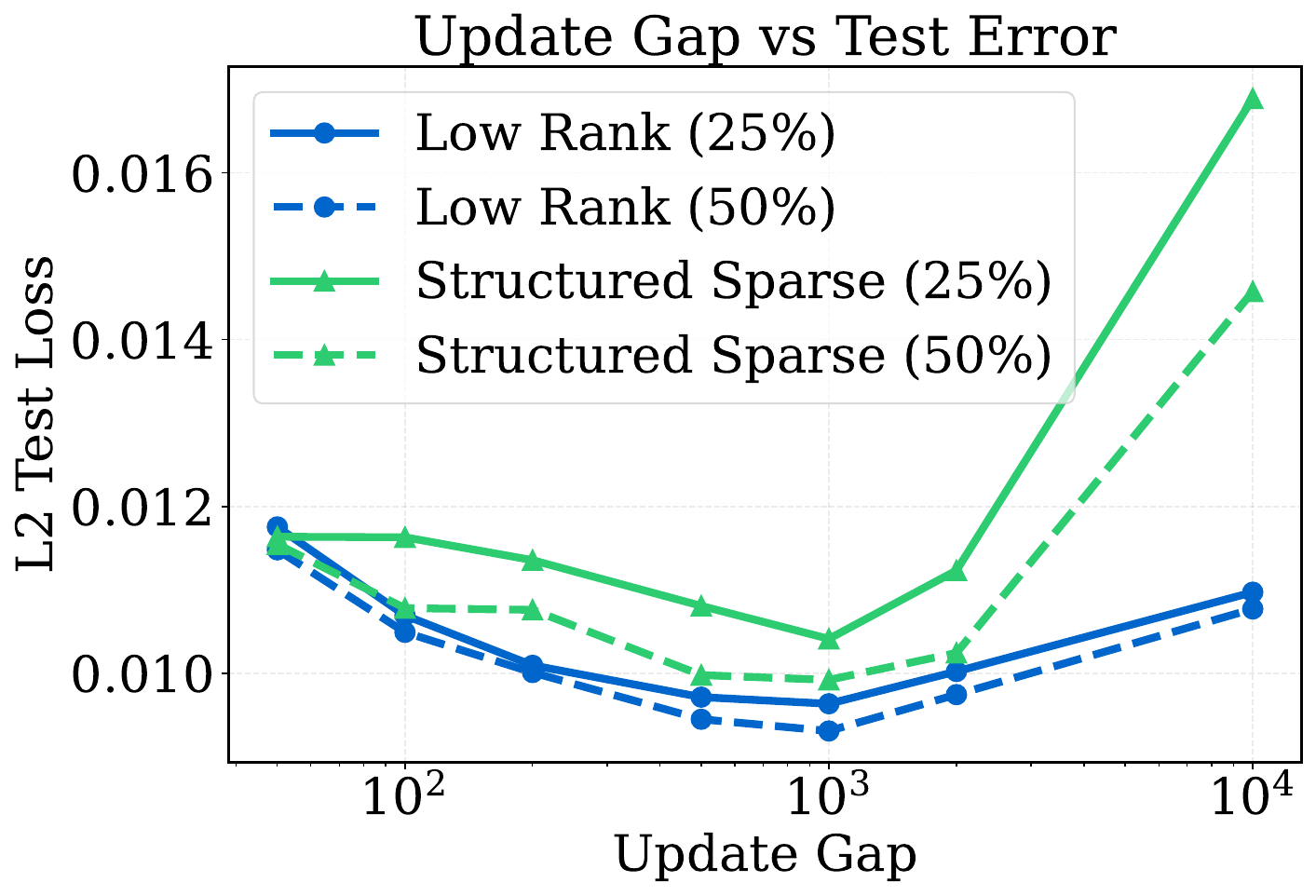}
    \end{subfigure}
    \caption{Left: Low-rank, structured, at various rank percentages vs \method\ at $25\%$ and baseline Adam optimizer. Right: Effect of projection gap on $L_2$ test loss for sparse and low-rank projections at 25\% and 50\% compression, trained for 150 epochs on the NS128 dataset.}
    \label{fig:gap_ablation_plot}
\vspace{-5mm}
\end{figure}

\subsection{Tensor Low-rank vs matrix GaLore}
\label{sec_app:matricizing_vs_tensor}

We highlight that our higher-order low-rank factorization version of \method\ shows superior performance to the direct extension of GaLore as described in~\ref{sec:naive_galore}. In Tab.~\ref{tab:2}), we show results for Darcy Flow comparing a higher-order low-rank factorization to construct the projection matrix instead of naively unfolding the tensor to 2d and removing the inherent structure. Comparing different unfoldings (variations of $d$), we observed up to a 48\% improvement in test loss with a rank of 0.25, while reducing the optimizer state memory from $2.09$GB to $0.5$GB.

\begin{table}
\centering
\caption{Model performance on Darcy-flow.}
\begin{tabular}{lccccccc}
\hline
Model & \multicolumn{6}{c|}{Test Loss (1e-2) at Rank Ratio} & Gain (\%) \\
& 0.01 & 0.1 & 0.25 & 0.5 & 0.75 & 1.0 & \\
\hline
FNO Baseline & - & - & - & - & - & 0.205 & / \\
FNO - Tensor Low-rank & \textbf{0.147} & \textbf{0.108} & \textbf{0.105} & \textbf{0.107} & \textbf{0.140} & \textbf{0.173} & \textbf{49} \\
FNO - GaLore (d=1) & 0.256 & 0.232 & 0.212 & 0.245 & 0.201 & 0.190 & 8 \\
FNO - GaLore (d=2) & 0.203 & 0.192 & 0.168 & 0.178 & 0.170 & 0.180 & 19 \\
FNO - GaLore (d=3) & 0.234 & 0.212 & 0.201 & 0.193 & 0.196 & 0.182 & 11 \\
\hline
\label{tab:2}
\end{tabular}
\end{table}

\begin{table}[h]
\centering
\caption{Evaluating low-rank tensor factorization across various tasks. }
\begin{tabular}{lcccccc}
\hline
Model & Rank & Memory & Train & Test $H_1$ & Test $L_2$  & Gain \\
& Ratio & (GB) & (Loss ($\times 10^{-2}$)) & (Loss ($\times 10^{-2}$)) & (Loss ($\times 10^{-2}$)) &(\%) \\
\hline
\hline
{\textbf{Darcy}} \\
Baseline & 1.0 & 8.88 & 0.7151 & 1.6230 & 0.2050 & / \\
GaLore (d=2) & 0.25 & 7.34 & 0.4200 & 1.3210 & 0.1680 & 19 \\
\textbf{Tensor low-rank} & 0.25 & 7.32 & \textbf{0.2930} & \textbf{0.8680} & \textbf{0.1050} & \textbf{48.8} \\
\hline
\multicolumn{5}{l}{\textbf{Navier-Stokes}} & \\
Baseline & 1.0 & 77 & 1.0630 & 1.9010 & 0.6152 & / \\
GaLore (d=1) & 0.5 & 68 & 4.3340 & 5.5830 & 1.9952 & -223 \\
\textbf{Tensor low-rank} & 0.5 & 55 & 1.2340 & 2.0850 & 0.6480 &  -5.4\\

\hline
\multicolumn{5}{l}{\textbf{ElectroMagnetic}} & \\
Baseline & 1.0 & 4.83 & 2.973 & 0.1902 & 0.2000  & / \\
GaLore (d=2) & 0.25 & 4.83 & 2.392 & 0.1802 & 0.1900 & 5 \\
\textbf{Tensor low-rank} & 0.25 & 4.63 & \textbf{2.132} & \textbf{0.1681}  & \textbf{0.1782} & \textbf{11} \\
\hline
\multicolumn{5}{l}{\textbf{Burgers}} & \\
Baseline & 1.0 & 3.94 &0.0064& 0.0050 & 0.0026 & / \\
GaLore (d=2) & 0.5 & 3.88 & 0.0052 & 0.0100 & 0.0062 &  -250 \\
\textbf{Tensor low-rank} & 0.5 & 3.87 & \textbf{0.0026} &\textbf{0.0041} &\textbf{0.0025} & \textbf{+5} \\
\hline
\label{tab:1}
\end{tabular}
\end{table}

\label{sec:matrix}
We evaluate three approaches to matricizing a tensor gradient with shape $C_{in} \times C_{out} \times M_x \times M_y$. The first, which we call "rollout=1", combines the last 3 dimensions into one matrix dimension, resulting in a matrix of shape $C_{in} \times (C_{out} * M_x * M_y)$. The second, "rollout=2", combines the first two dimensions into the first matrix dimension and the last two dimensions into the second matrix dimension, resulting in a matrix of shape $(C_{in} * C_{out}) \times (M_x * M_y)$. The last, "rollout=3", combines the last three dimensions into the second matrix dimension, resulting in a matrix of shape $C_{in} \times (C_{out} * M_x * M_y)$. We showcase results and comparisons for all three approaches in Tab.~\ref{reshape_abl}.

\label{app:abl}
All of the subsequent results are with varying rank ratios on the \textbf{\method} method for all datasets. We report both the training and testing loss/accuracy.
\begin{table}[!h]
\centering
\caption{Model performance on Burgers}
\label{tab:burgers_}
\begin{tabular}{lccccr}
\hline
Model & Rank Ratio & \makecell{Train Loss (1e-4)} & \makecell{Test Loss(1e-4)} & Gain (\%) \\
\hline
FNO Baseline & Full Rank & 0.205 & 0.262 & / \\
FNO - \textbf{Tensor low-rank} & 0.1 & 0.115 & 0.321 & -19 \\
FNO - \textbf{Tensor low-rank} & 0.25 & 0.095 & 0.271 & -4 \\
FNO - \textbf{Tensor low-rank} & 0.5 & 0.086 & 0.253 & +5 \\
FNO - \textbf{Tensor low-rank} & 0.75 & 0.083 & 0.246 & +8 \\
FNO - \textbf{Tensor low-rank} & 1.00 & 0.083 & \textbf{0.242} & \textbf{+9} \\
\hline
\end{tabular}
\end{table}

\begin{table}[!h]
\centering
\label{tab:22}
\caption{Model performance on Darcy-flow}
\begin{tabular}{lccccr}
\hline
Model & Rank Ratio & \makecell{Train Loss (1e-2)} & \makecell{Test Loss(1e-2)} & Gain (\%)\\
\hline
FNO Baseline & Full Rank & 0.715 & 0.205 & / \\
FNO - \textbf{Tensor low-rank} & 0.01 & 0.465 & 0.147 & +30  \\
FNO - \textbf{Tensor low-rank} & 0.1 & 0.323 & 0.108 & +48 \\
FNO - \textbf{Tensor low-rank} & 0.25 & \textbf{0.293} & \textbf{0.105} & \textbf{+49}\\
FNO - \textbf{Tensor low-rank} & 0.5 & 0.275 & 0.107 & +49\\
FNO - \textbf{Tensor low-rank} & 0.75 & 0.379 & 0.140 & +40 \\
FNO - \textbf{Tensor low-rank} & 1.00 & 0.715 & 0.173 & +16  \\
\hline
\end{tabular}
\end{table}

\begin{table}[!h]
\centering
\caption{Model performance on EM.}
\label{tab:EM_}
\begin{tabular}{lccccccc}
\hline
Model & \multicolumn{6}{c|}{Test Loss (1e-2) at Rank Ratio} & Gain (\%) \\
& 0.01 & 0.1 & 0.25 & 0.5 & 0.75 & 1.0 & \\
\hline
FNO Baseline & - & - & - & - & - & 0.200 & / \\
FNO - \textbf{Tensor low-rank} & 0.187 & 0.185 & \textbf{0.178} & 0.176 & 0.174 & 0.206 & \textbf{11} \\
FNO - GaLore (d=1) & 0.213 & 0.192 & 0.193 & 0.189 & 0.194 & 0.200 & 7 \\
FNO - GaLore (d=2) & 0.205 & 0.206 & 0.195 & 0.196 & 0.201 & 0.199 & 3\\
\hline
\label{tab:12}
\end{tabular}
\end{table}

\begin{table}[!h]
\centering
\caption{Ablation: GaLore and \textbf{Tensor low-rank} Rank Comparison}
\begin{tabular}{lccccr}
\hline
Method & \% orig. parameters & \makecell{GaLore Test $L_2$ ($\times 10^{-2}$)} & \makecell{\textbf{Tensor low-rank} Test $L_2$ ($\times 10^{-2}$)}\\
\hline
GaLore (d=1) & 25 & $2.495 \pm 0.920$ & \textbf{0.9141} $\pm 0.0064$ \\
 & 50 & $3.594 \pm 0.885$ & \textbf{0.7622} $\pm 0.0984$\\
 & 75 & $3.298 \pm  1.96$ & \textbf{0.6697} $\pm 0.0746$\\
\hline
GaLore (d=2) & 25 & $8.715 \pm 0.252$ & \textbf{0.9141} $\pm 0.0064$ \\
 & 50 & $8.683 \pm 0.0014$ & \textbf{0.7622} $\pm 0.0984$\\
 & 75 & $8.950 \pm 0.0141$ & \textbf{0.6697} $\pm 0.0746$\\
\hline
GaLore (d=3) & 25 & $8.723 \pm 0.0149$ & \textbf{0.9141} $\pm 0.0064$ \\
 & 50 & $8.702 \pm 0.0108$ & \textbf{0.7622} $\pm 0.0984$\\
 & 75 & $8.585 \pm 0.0171$ & \textbf{0.6697} $\pm 0.0746$\\
\hline
\end{tabular}
\label{reshape_abl}
\end{table}

\subsection{Structured Sparsity Mask Strategy}
\label{sec_app:structured_sparsity}

Tab.~\ref{tab:sparse_ablation} shows an ablation of three structured sparse projection strategies: \textbf{Top-$k$}, \textbf{Probability}, and \textbf{Rand-$k$}, evaluated at 25\% and 50\% sparsity.

We find that \textbf{Rand-$k$} consistently outperforms both Top-$k$ and Probability sampling, especially at higher sparsity (50\%). This aligns with prior findings in GRASS~\cite{2024grass}.

The performance gap stems from how each method selects entries:
\begin{itemize}
    \item \textbf{Top-$k$} selects slices with the highest norms.
    \item \textbf{Probability} sampling assigns selection probabilities proportional to slice norms.
    \item \textbf{Rand-$k$} uniformly samples slices at random, without relying on norm computations.
\end{itemize}

Both Top-$k$ and Probability flatten the tensor by computing slice-wise norms, discarding its multi-dimensional structure. In contrast, Rand-$k$ respects the original tensor layout, which helps preserve informative patterns and leads to better generalization.

\begin{table}[h]
\centering
\caption{Ablation of sparse projection methods at 25\% and 50\% sparsity. Train and test losses are reported as $L_2$ loss ($\times 10^{-3}$) with mean ± std over multiple runs. }
\begin{tabular}{lcccc}
\hline
Method &  Train loss & Test $L_2$ \\
\hline
\hline
\textbf{Top-$k$} \\
25\%     & \num{12.05} & \num{6.50} \\
50\%     & \num{12.92} & \num{7.93} \\
\hline
\textbf{Probability} \\
25\%     & \num{10.91} & \num{6.20} \\
50\%     & \num{9.59} & \num{6.07} \\
\hline
\textbf{Rand-$k$} \\
25\%     & \num{12.01} & \num{6.22} \\
50\%     & $\mathbf{9.14}$ & $\mathbf{5.54}$ \\
\hline
\label{tab:sparse_ablation}
\end{tabular}
\end{table}

\begin{table}[!h]
\centering
\caption{Model performance on EM}
\begin{tabular}{lccccr}
\hline
Model & Rank Ratio & \makecell{Train Loss} & \makecell{Test Loss} & Gain (\%) \\
\hline
Complex FNO Baseline & Full Rank & 2.973 & 0.200  & / \\
Complex FNO - \textbf{Tensor low-rank} & 0.01 & 4.198 & 0.249 & -20 \\
Complex FNO - \textbf{Tensor low-rank} & 0.1 & 2.936 & 0.217 & -8 \\
Complex FNO - \textbf{Tensor low-rank} & 0.25 & \textbf{2.132} & \textbf{0.178} & +11 \\
Complex FNO - \textbf{Tensor low-rank} & 0.5 & {2.430} & {0.184} & {+8} \\
Complex FNO - \textbf{Tensor low-rank} & 0.75 & 2.719 & 0.192 & +4 \\
Complex FNO - \textbf{Tensor low-rank} & 1.00 & {2.397} & 0.185 & +8 \\
\hline
\end{tabular}
\end{table}
\section{Architecture and Training Details}
\paragraph{Sobolev Loss for PDE Training.}
In training NOs for PDEs we employ both the $L^2$ and Sobolev $H^1$ losses to provide a comprehensive assessment of model performance. While the $L^2$ loss measures point-wise accuracy of predictions, the $H^1$ loss, defined as $\|u - \hat{u}\|_{H^1}^2 = \|u - \hat{u}\|_{L^2}^2 + \|\nabla u - \nabla \hat{u}\|_{L^2}^2$, accounts for both the function values and their gradients. This is particularly crucial for PDEs, as it ensures that the learned solutions not only match the target values but also preserve the smoothness and differential properties inherent in the physical systems being modeled. 

\paragraph{Sobolev Loss for Complex Wave Phenomena.}
The Sobolev $H^1$ loss proves especially valuable when dealing with complex wave phenomena, as demonstrated in our experiments with the EM Dataset using Complex-FNOs. In this case, the $H^1$ loss not only measures the accuracy of the predicted complex electric field envelope but also ensures that its spatial derivatives are correctly captured. This is crucial for accurately representing the rapid oscillations and sharp peaks characteristic of EM waves. Our results show that \textbf{Tensor low-rank} with a rank ratio of 0.25 achieved an 11\% improvement in overall test loss compared to the baseline, with the $H^1$ loss decreasing from 0.1902 to 0.1681. This improvement is particularly significant given the challenging nature of the EM dataset, which involves predicting the complex electric field envelope resulting from nonlinear interactions in waveguides. The enhanced performance in $H^1$ loss indicates that our model not only matches the amplitude of the EM waves more accurately but also better captures the rapid spatial variations and peak formations. This is critical in applications such as optical pulse propagation, where precise modeling of field gradients and peak intensities is essential for predicting phenomena like second-harmonic generation and phase matching.
\begin{table}[!h]
\centering
\small
\begin{tabular}{|p{1.5cm}|p{1.5cm}|p{7cm}|p{2.5cm}|}
\hline
\textbf{Dataset} & \textbf{Model} & \textbf{Architecture Details} & \textbf{Optimizer \& Scheduler} \\
\hline
Burgers & FNO & 
\begin{itemize}[nosep]
    \item 4 layers, 90 modes
    \item 256 hidden channels, 256 projection channels
    \item Skip Connections: 'linear'
    \item Positional embedding: 'grid'
\end{itemize} & 
Adam with step LR $3e-4$, weight decay $2e-6$ 500 epochs, batch size 16. Trained with $H_1$ loss.\\
\hline
NS128 & FNO & 
\begin{itemize}[nosep]
    \item 4 layers, 64 x 64 modes
    \item 64 hidden channels, 256 projection channels
    \item Skip: 'linear'
    \item Use channel MLP: 1
    \item Channel MLP expansion: 0.5, dropout: 0
\end{itemize} & 
Complex Adam with step LR 3e-4, weight decay 1e-4, 500 epochs, batch size 8.  Update decomposition frequency: 1000. Trained with $H_1$ loss.\\ 
\hline
NS1024 - max memory test & FNO & 
\begin{itemize}[nosep]
    \item 4 layers, 128 modes
    \item 255 hidden channels, 256 projection channels
    \item Skip: 'linear'
    \item Channel MLP expansion: 0.5, dropout: 0
\end{itemize} & 
Complex Adam with step LR 5e-3, weight decay 1e-4, 100 epochs in total: batch size 8 for 50 iterations and resolution 256, then batch size 2 for resolution 1024. Update decomposition frequency: 500. Trained with $L2$ loss. \\
\hline
NS1024 $\text{Re}=10^5$ & FNO & 
\begin{itemize}[nosep]
    \item 4 layers, 128 modes
    \item 128 hidden channels, 256 projection channels
    \item Skip: 'linear'
    \item Channel MLP expansion: 0.5, dropout: 0
\end{itemize} & 
Complex Adam with step LR 5e-3, weight decay 1e-4, 100 epochs in total: batch size 8 for 50 iterations and resolution 256, then batch size 2 for resolution 1024. Update decomposition frequency: 500. Trained with $L2$ loss.\\ 
\hline
Darcy Flow & FNO & 
\begin{itemize}[nosep]
    \item 4 layers, 64 modes
    \item 128 hidden channels, 128 projection channels
    \item Skip: 'linear'
\end{itemize} & 
Adam with step LR $1e-3$, weight decay $1e-4$, 250 epochs, batch size 2. Trained with $L_2$ loss.\\
\hline
EM Wave & Complex-FNO & 
\begin{itemize}[nosep]
    \item 8 layers, 128 modes
    \item 128 hidden channels, 128 projection channels
    \item Skip: 'linear'
    \item Complex data: True
    \item Complex activation function: True
\end{itemize} & 
Complex Adam with step LR 1e-4, weight decay 2e-6, batch size 32, 1000 epochs. Trained with $H_1$ loss.\\
\hline
\end{tabular}
\caption{Detailed FNO Architecture Specifications for Different Datasets}
\label{tab:detailed_fno_arch}
\end{table}

\begin{section}{Slowdown in Training}

While \textbf{Tensor low-rank} does introduce additional computational overhead from the tensor decomposition step, we have carefully analyzed the impact on training speed and efficiency. Our experiments have shown that the memory savings achieved by \textbf{Tensor low-rank} often outweigh the slight increase in computational cost, resulting in an overall improvement in training time and resource utilization.
Specifically, we have measured the training time for \textbf{Tensor low-rank} compared to the baseline FNO model and the GaLore approach. Our results indicate that the slowdown in training time is modest, typically in the range of 5-20\%, depending on the dataset and model configuration. This is a reasonable trade-off given the significant memory savings (up to 75\% reduction in optimizer memory) that \textbf{Tensor low-rank} provides.

\begin{table}[!h]
\centering
\begin{tabular}{lccc}
\hline
\textbf{Model} & \textbf{Rank} & \textbf{Time/epoch(s)} & \textbf{Slowdown (\%)} \\
\hline
Baseline & 1.0 & 34.96 & -- \\
GaLore & 0.20 & 34.47 & -1.40 \\
GaLore & 0.25 & 34.79 & -0.48 \\
GaLore & 0.50 & 36.27 & 3.75 \\
GaLore & 0.75 & 37.50 & 7.26 \\
\textbf{Tensor low-rank} (40, 40, 40, 24) & 0.20 & 36.53 & 5.98 \\
\textbf{Tensor low-rank} (48, 48, 48, 24) & 0.25 & 38.30 & 10.08 \\
\textbf{Tensor low-rank} (56, 56, 56, 24) & 0.50 & 40.63 & 12.03 \\
\textbf{Tensor low-rank} (64, 64, 56, 32) & 0.75 & 44.93 & 19.84 \\
\hline
\end{tabular}
\caption{Comparison of model execution times, ranks, and relative slowdown}
\label{tab:model_comparison}
\end{table}

Moreover, we have incorporated techniques such as "warm-restart" initialization of the tensor decomposition to amortize the computational overhead across training iterations. This helps minimize the impact on the overall training efficiency. We have also explored opportunities to further optimize the tensor decomposition computations, which could potentially reduce the training time slowdown even further.

\end{section}
\begin{remark}[Real-Valued Analysis]
For clarity of presentation, we develop the theory of \textbf{Tensor low-rank} assuming all tensors are real-valued, i.e., $\mathcal{W}_l, \mathcal{G}_t \in \mathbb{R}^{N_1 \times N_2 \times N_3 \times N_4}$ and all associated operations are in real space. This simplification allows us to focus on the core geometric and algebraic properties without the additional complexity of complex conjugates and Hermitian operations. The extension to complex-valued tensors (as needed for Fourier Neural Operators where weights may be complex in the frequency domain) is straightforward: inner products become Hermitian inner products, transposes become conjugate transposes, and orthogonality conditions incorporate complex conjugates. All main results remain valid with these natural modifications.
\end{remark}

\section{Parameter Complexity Analysis}

\begin{figure}
    \centering
    \includegraphics[width=0.9\linewidth]{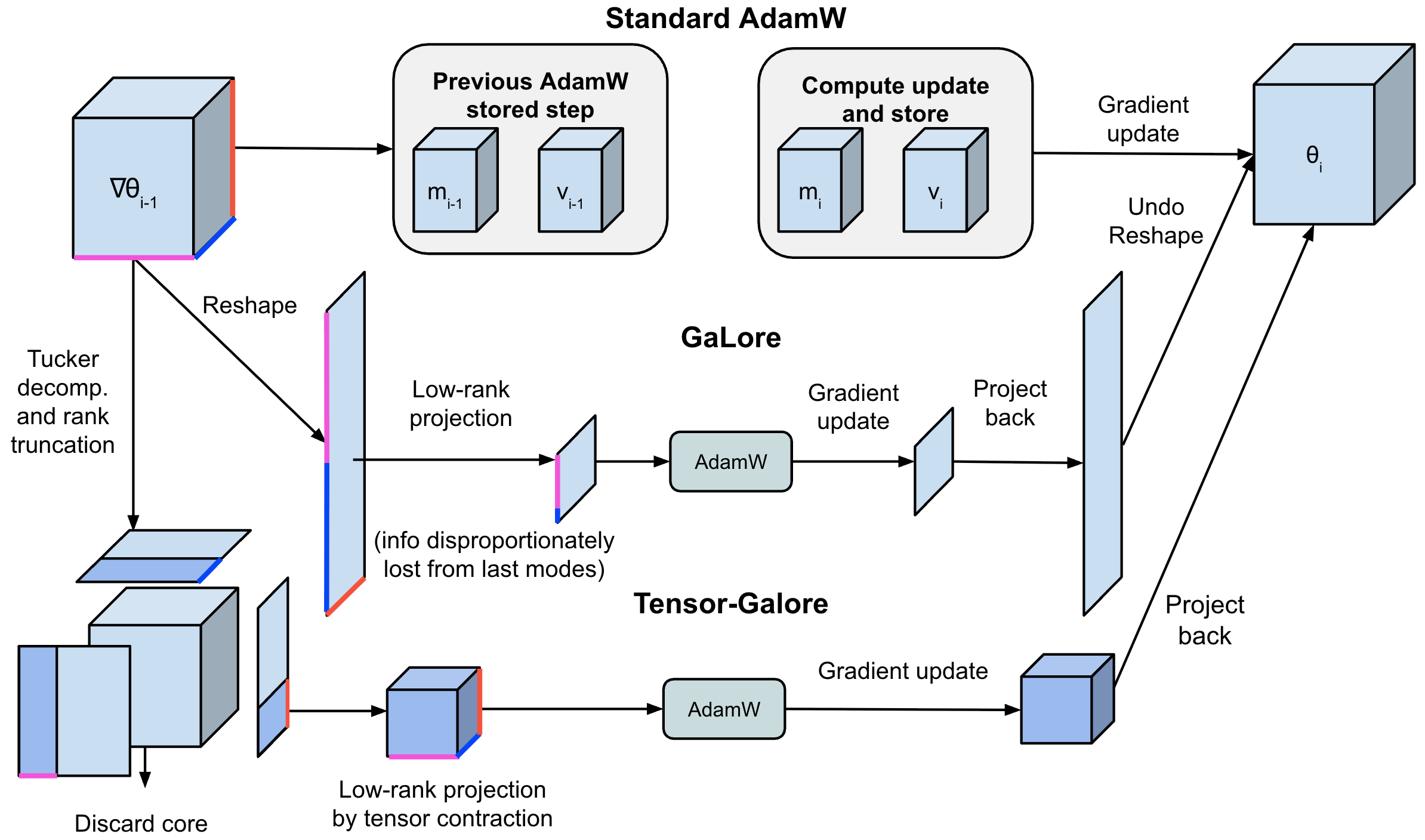}
    \caption{Comparison of the higher-order low-rank decomposition Tensor Low-Rank with standard Adam and GaLore. GaLore applies matrix-based low-rank projection after reshaping tensors. Our \textbf{Tensor low-rank} method leverages tensor decomposition to perform low-rank projection directly on tensor gradients, preserving multidimensional structure.}
    \label{fig:old_main1}
\end{figure}

\begin{algorithm}
\caption{Adam with low-rank tensor decomposition}
\begin{algorithmic}[1]
\Require A layer weight tensor $\mathcal{W} \in \mathbb{C}^{N_1 \times N_2 \times N_3 \times N_4}$. Step size $\eta$, scale factor $\alpha$, decay rates $\beta_1, \beta_2$, rank $r$, subspace change frequency $T$.
\State Initialize first-order moment $\mathcal{M}_0 \in \mathbb{C}^{r \times r \times r \times r} \gets 0$
\State Initialize second-order moment $\mathcal{V}_0 \in \mathbb{C}^{r \times r \times r \times r} \gets 0$
\State Initialize step $t \gets 0$
\Repeat
    \State $\mathcal{G}_t \in \mathbb{C}^{N_1 \times N_2 \times N_3 \times N_4} \gets -\nabla_\mathcal{W} \phi_t(\mathcal{W}_t)$
    \If{$t \bmod T = 0$}
        \State $\mathcal{C}, \{U^{(n)}\}_{n=1}^4 \gets \text{Projection}(\mathcal{G}_t, \text{rank}=r)$ instead of tucker
    \Else
        \State $\mathcal{C}, \{U^{(n)}\}_{n=1}^4 \gets \mathcal{C}_{t-1}, \{U^{(n)}_{t-1}\}_{n=1}^4$ \Comment{Reuse the previous projector.}
    \EndIf
    \State $\mathcal{R}_t \gets \mathcal{G}_t \times_1 {U^{(1)}}^\top \times_2 {U^{(2)}}^\top \times_3 {U^{(3)}}^\top \times_4 {U^{(4)}}^\top$ \Comment{Project gradient into compact space.}
    \State \textbf{UPDATE}($\mathcal{R}_t$) by Adam:
    \State \quad $\mathcal{M}_t \gets \beta_1 \cdot \mathcal{M}_{t-1} + (1 - \beta_1) \cdot \mathcal{R}_t$
    \State \quad $\mathcal{V}_t \gets \beta_2 \cdot \mathcal{V}_{t-1} + (1 - \beta_2) \cdot |\mathcal{R}_t\Bar{\mathcal{R}_t}|$ \Comment{We use the complex conjugate update.}
    \State \quad $\mathcal{M}_t \gets \mathcal{M}_t / (1 - \beta_1^t)$
    \State \quad $\mathcal{V}_t \gets \mathcal{V}_t / (1 - \beta_2^t)$
    \State \quad $\mathcal{N}_t \gets \mathcal{M}_t / (\sqrt{\mathcal{V}_t} + \epsilon)$
    \State $\tilde{\mathcal{G}}_t \gets \alpha \cdot \mathcal{N}_t \times_1 U^{(1)} \times_2 U^{(2)} \times_3 U^{(3)} \times_4 U^{(4)}$ \Comment{Project back to original space.}
    \State $\mathcal{W}_t \gets \mathcal{W}_{t-1} + \eta \cdot \tilde{\mathcal{G}}_t$
    \State $t \gets t + 1$
\Until{convergence criteria met.}
\State \Return $\mathcal{W}_t$
\end{algorithmic}
\end{algorithm}

\label{sec:memory}
To understand the theoretical advantages of \textbf{Tensor low-rank} over matrix-based GaLore, we provide a detailed analysis of the parameter complexity for both approaches. This analysis demonstrates why tensor decomposition leads to more efficient memory usage while maintaining expressiveness.

\subsection{Memory Analysis}

We provide a theoretical analysis of the memory requirements for \textbf{Tensor low-rank} compared to baseline methods and matrix GaLore variants. Consider a weight tensor $W \in \mathbb{C}^{N_1 \times N_2 \times N_3 \times N_4}$ in a FNO Spectral layer. Tab.~\ref{tab:memory_requirements} summarizes the memory requirements for different methods. The baseline approach stores the full tensor and its corresponding optimizer states. For a rank ratio $r$ $(0 < r \leq 1)$, \textbf{Tensor low-rank} requires storing the factor matrices, resulting in substantial memory savings, especially for the optimizer states. In this table, we assume the use of a complex-valued Adam optimizer, which typically requires two additional tensors (first and second moments) for each parameter.

\begin{table}[h]
\centering
\caption{Theoretical memory requirements for different methods}
\label{tab:memory_requirements}
\begin{tabular}{lcc}
\hline
Method & Weight Parameters & Optimizer States (Adam)\\
\hline
Baseline & $N_1N_2N_3N_4$ & $2N_1N_2N_3N_4$ \\
Matrix GaLore (rollup dim $1$) & $N_1N_2N_3N_4$ & $2r(N_1 + N_2N_3N_4)$ \\
\textbf{Tensor low-rank} (Tucker) & $N_1N_2N_3N_4$& $2r(N_1+N_2+N_3+N_4)$ \\
\hline
\end{tabular}
\end{table}

\subsubsection{Problem Setup}

Consider a 4D tensor weight $\mathcal{W} \in \mathbb{R}^{I_1 \times I_2 \times I_3 \times I_4}$ from a Fourier Neural Operator layer, where:
\begin{itemize}
    \item $(I_1, I_2)$ correspond to input/output channels
    \item $(I_3, I_4)$ correspond to spatial frequency modes
\end{itemize}

\subsubsection{Matrix-based Approach (GaLore)}

In the matrix-based GaLore approach, we must first reshape the tensor into a matrix. There are several possible matricization strategies:

\begin{enumerate}
    \item $\mathbf{W}_{(1)} \in \mathbb{R}^{I_1 \times (I_2I_3I_4)}$ and all permutations of $\mathbf{W}_{(i)}$
    \item $\mathbf{W}_{(12)} \in \mathbb{R}^{(I_1I_2) \times (I_3I_4)}$
\end{enumerate}

For a rank-$R$ SVD approximation of the matricized tensor:

\begin{equation}
    \mathbf{W} \approx \mathbf{U}\mathbf{\Sigma}\mathbf{V}^H
\end{equation}

The parameter count for storing the low-rank factors is:
\begin{itemize}
    \item For $\mathbf{W}_{(1)}$: $R(I_1 + I_2I_3I_4)$ parameters
    \item For $\mathbf{W}_{(12)}$: $R(I_1I_2 + I_3I_4)$ parameters
\end{itemize}

\subsubsection{Tensor-based Approach (\textbf{Tensor low-rank})}

In \textbf{Tensor low-rank}, we use Tucker decomposition with ranks $(R_1, R_2, R_3, R_4)$:

\begin{equation}
    \mathcal{W} \approx \mathcal{G} \times_1 \mathbf{U}^{(1)} \times_2 \mathbf{U}^{(2)} \times_3 \mathbf{U}^{(3)} \times_4 \mathbf{U}^{(4)}
\end{equation}

where:
\begin{itemize}
    \item $\mathcal{G} \in \mathbb{R}^{R_1 \times R_2 \times R_3 \times R_4}$ is the core tensor
    \item $\mathbf{U}^{(n)} \in \mathbb{R}^{I_n \times R_n}$ are the factor matrices
\end{itemize}

The total parameter count is:
\begin{equation}
    P_{Tucker} = R_1R_2R_3R_4 + \sum_{n=1}^4 I_nR_n
\end{equation}

\subsubsection{Comparative Analysis}

Let's consider a practical case where:
\begin{itemize}
    \item $N = I_1 = I_2$ (equal input/output channels)
    \item $M = I_3 = I_4$ (equal spatial dimensions)
    \item For Tucker: $r_{max} = R_1 = R_2 = R_3 = R_4$ (equal ranks)
    \item For matrix SVD: $R = r_{max}^2$ (equivalent rank)
\end{itemize}

Then:

\begin{enumerate}
    \item Matrix GaLore (best case):
    \begin{equation}
        P_{Matrix} = r_{max}^2(N^2 + M^2)
    \end{equation}

    \item \textbf{Tensor low-rank}:
    \begin{equation}
        P_{Tensor} = r_{max}^4 + 2r_{max}N + 2r_{max}M
    \end{equation}
\end{enumerate}

In typical neural operator architectures:
\begin{itemize}
    \item $N \gg r_{max}$ (number of channels much larger than rank)
    \item $M \gg r_{max}$ (spatial dimensions much larger than rank)
\end{itemize}

Therefore:
\begin{itemize}
    \item Matrix case complexity: $O(r^2(N^2 + M^2))$
    \item Tensor case complexity: $O(N + M + r^4)$
\end{itemize}

\subsubsection{Memory Savings Analysis}

For concrete numbers, consider a typical FNO layer with:
\begin{itemize}
    \item $N = 64$ channels
    \item $M = 128$ modes
    \item $r_{max} = 16$ (rank)
\end{itemize}

Matrix GaLore parameters:
\begin{equation}
    P_{Matrix} = 256(64^2 + 128^2) \approx 5.2M
\end{equation}

\textbf{Tensor low-rank} parameters:
\begin{equation}
    P_{Tensor} = 65,536 + 2(16)(64) + 2(16)(128) \approx 70K
\end{equation}

This represents a $\sim$75x reduction in parameter count, which directly translates to memory savings in the optimizer states. The savings become even more pronounced as the spatial dimensions ($M$) increase, which is crucial for high-resolution problems.

\subsubsection{Impact on Expressiveness}

Despite the significant reduction in parameters, \textbf{Tensor low-rank} maintains expressiveness because:
\begin{enumerate}
    \item The Tucker decomposition preserves the natural tensor structure of the operator
    \item Each mode has its own rank parameter, allowing for more flexible approximation
    \item The core tensor captures higher-order interactions between modes
\end{enumerate}

This explains why \textbf{Tensor low-rank} can achieve comparable or better performance while using significantly less memory than matrix-based approaches.

\section{Tensor Operations and Notation}
\label{sec:tensor1}
\begin{definition}[Tensor]
An order-d tensor $\mathcal{A} \in \mathbb{R}^{I_1 \times I_2 \times \cdots \times I_d}$ is a d-dimensional array with entries $a_{i_1,i_2,\ldots,i_d}$, where $1 \leq i_k \leq I_k$ for $k = 1,\ldots,d$.
\end{definition}

\begin{definition}[Mode-k Unfolding]
The mode-k unfolding of tensor $\mathcal{A}$, denoted as $\mathcal{A}_{(k)} \in \mathbb{R}^{I_k \times (I_1\cdots I_{k-1}I_{k+1}\cdots I_d)}$, arranges the mode-k fibers as columns of the resulting matrix. Specifically:
\[
(\mathcal{A}_{(k)})_{i_k,j} = a_{i_1,\ldots,i_d}
\]
where $j = 1 + \sum_{m=1,m\neq k}^d (i_m-1)\prod_{n=1,n\neq k}^{m-1} I_n$.
\end{definition}

\begin{definition}[Mode-k Product]
The mode-k product of a tensor $\mathcal{A} \in \mathbb{R}^{I_1 \times \cdots \times I_d}$ with a matrix $U \in \mathbb{R}^{J \times I_k}$, denoted as $\mathcal{A} \times_k U$, results in a tensor $\mathcal{B} \in \mathbb{R}^{I_1 \times \cdots \times I_{k-1} \times J \times I_{k+1} \times \cdots \times I_d}$ with entries:
\[
(\mathcal{A} \times_k U)_{i_1,\ldots,i_{k-1},j,i_{k+1},\ldots,i_d} = \sum_{i_k=1}^{I_k} a_{i_1,\ldots,i_d} u_{j,i_k}
\]
\end{definition}

\begin{proposition}[Properties of Mode-k Product]
For a tensor $\mathcal{A}$ and matrices $U, V$ of appropriate sizes:
\begin{enumerate}
    \item $(U \times_k \mathcal{A})_{(k)} = U\mathcal{A}_{(k)}$
    \item $\mathcal{A} \times_k U \times_l V = \mathcal{A} \times_l V \times_k U$ for $k \neq l$
    \item $\mathcal{A} \times_k U \times_k V = \mathcal{A} \times_k (VU)$
\end{enumerate}
\end{proposition}

\begin{definition}[Tensor Inner Product]
The inner product of two tensors $\mathcal{A}, \mathcal{B} \in \mathbb{R}^{I_1 \times \cdots \times I_d}$ is:
\[
\langle \mathcal{A}, \mathcal{B} \rangle = \sum_{i_1=1}^{I_1} \cdots \sum_{i_d=1}^{I_d} a_{i_1,\ldots,i_d}\overline{b_{i_1,\ldots,i_d}}
\]
\end{definition}

\begin{definition}[Tensor Norms]
For a tensor $\mathcal{A}$:
\begin{enumerate}
    \item Frobenius norm: $\|\mathcal{A}\|_F = \sqrt{\langle \mathcal{A}, \mathcal{A} \rangle}$
    \item Mode-k spectral norm: $\|\mathcal{A}\|_{(k)} = \|\mathcal{A}_{(k)}\|_2$
    \item Spectral norm: $\|\mathcal{A}\| = \max_{\|x^{(k)}\|=1} \|\mathcal{A} \times_1 x^{(1)} \times_2 \cdots \times_d x^{(d)}\|$
\end{enumerate}
\end{definition}

\begin{definition}[Tensor Outer Product]
The outer product of vectors $u^{(k)} \in \mathbb{R}^{I_k}$ for $k=1,\ldots,d$ is a tensor $\mathcal{A} = u^{(1)} \circ u^{(2)} \circ \cdots \circ u^{(d)}$ with entries:
\[
a_{i_1,\ldots,i_d} = u^{(1)}_{i_1}u^{(2)}_{i_2}\cdots u^{(d)}_{i_d}
\]
\end{definition}

\begin{definition}[Tensor Contraction]
The contraction of a tensor $\mathcal{A} \in \mathbb{R}^{I_1 \times \cdots \times I_d}$ along modes $p$ and $q$ (where $I_p = I_q$) is:
\[
(\text{Contract}_{p,q}(\mathcal{A}))_{i_1,\ldots,i_{p-1},i_{p+1},\ldots,i_{q-1},i_{q+1},\ldots,i_d} = \sum_{i=1}^{I_p} a_{i_1,\ldots,i_{p-1},i,i_{p+1},\ldots,i_{q-1},i,i_{q+1},\ldots,i_d}
\]
\end{definition}

\subsection{Tensor Trace and Inner Products}

\begin{definition}[Tensor Inner Product]
For tensors $\mathcal{A}, \mathcal{B} \in \mathbb{R}^{I_1 \times I_2 \times \cdots \times I_d}$, their inner product is:
\[
\langle \mathcal{A}, \mathcal{B} \rangle = \sum_{i_1=1}^{I_1} \sum_{i_2=1}^{I_2} \cdots \sum_{i_d=1}^{I_d} \mathcal{A}_{i_1,i_2,\ldots,i_d}\mathcal{B}_{i_1,i_2,\ldots,i_d}
\]
\end{definition}

\begin{definition}[Tensor Trace]
For a tensor $\mathcal{A}$, there are several equivalent ways to understand its trace:

1. Mode-wise trace:
\[
\text{tr}_k(\mathcal{A}) = \sum_{i_k=1}^{I_k} \mathcal{A}_{i_1,\ldots,i_k,\ldots,i_d}|_{i_k=i_k}
\]

2. Using mode-k unfolding:
\[
\text{tr}(\mathcal{A}_{(k)}) = \sum_{i=1}^{I_k} (\mathcal{A}_{(k)})_{i,i}
\]

3. Inner product interpretation:
When used in expressions like $\text{tr}(d\mathcal{W}_l^\top \times_1 X \times_2 Y)$, this is actually computing:
\[
\langle d\mathcal{W}_l, X \otimes Y \rangle
\]
\end{definition}

\begin{proposition}[Key Properties]
For the trace operation in tensor gradients:

1. Inner Product Form:
\[
\text{tr}(d\mathcal{W}^\top \times_1 X \times_2 Y) = \langle d\mathcal{W}, X \otimes Y \rangle
\]

2. Differential Form:
For scalar function $\phi$ and tensor $\mathcal{W}$:
\[
d\phi = \text{tr}(d\mathcal{W}^\top \times_1 X \times_2 Y) \implies \frac{\partial \phi}{\partial \mathcal{W}} = X \otimes Y
\]

3. Mode-wise Consistency:
\[
\text{tr}(d\mathcal{W}^\top \times_1 X \times_2 Y) = \text{tr}(X^\top d\mathcal{W}_{(1)}Y)
\]
where $d\mathcal{W}_{(1)}$ is the mode-1 unfolding.
\end{proposition}

\begin{example}
In the logsoftmax gradient computation:
\begin{align*}
-d\phi &= \text{tr}(d\mathcal{W}_l^\top \times_1 (P_1^\perp y)^\top \mathcal{J}_l \times_2 f_{l-1}^\top) \\
&= \langle d\mathcal{W}_l, \mathcal{J}_l^\top P_1^\perp y \otimes f_{l-1} \rangle
\end{align*}
This leads to the gradient term:
\[
\mathcal{G}_l = \mathcal{J}_l^\top P_1^\perp y \otimes f_{l-1}
\]
\end{example}

\begin{remark}[Connection to Matrix Case]
When working with matrices, the trace operation reduces to the familiar form:
\[
\text{tr}(A^\top B) = \langle A, B \rangle = \sum_{i,j} A_{ij}B_{ij}
\]
The tensor trace generalizes this to handle higher-order tensors while preserving the key property that it relates to directional derivatives through inner products.
\end{remark}

\subsection{Stable Rank for Tensors}

\begin{definition}[Matrix Stable Rank]
For a matrix $A$, the stable rank is defined as:
\[
\text{sr}(A) := \frac{\|A\|_F^2}{\|A\|_2^2}
\]
where $\|\cdot\|_F$ is the Frobenius norm and $\|\cdot\|_2$ is the spectral norm.
\end{definition}

\begin{definition}[Tensor Stable Rank]
For a non-zero tensor $\mathcal{T} \in \mathbb{R}^{N_1 \times N_2 \times ... \times N_d}$, we define the mode-wise stable rank vector as:
\[
\text{sr}(\mathcal{T}) = [\text{sr}_1(\mathcal{T}), \text{sr}_2(\mathcal{T}), ..., \text{sr}_d(\mathcal{T})]
\]
where for each mode $k$:
\[
\text{sr}_k(\mathcal{T}) := \frac{\|\mathcal{T}\|_F^2}{\|\mathcal{T}_{(k)}\|_2^2}
\]
Here:
\begin{itemize}
    \item $\mathcal{T}_{(k)}$ is the mode-k unfolding of tensor $\mathcal{T}$
    \item $\|\mathcal{T}\|_F^2 = \sum_{i_1,...,i_d} |\mathcal{T}_{i_1,...,i_d}|^2$ is the tensor Frobenius norm
    \item $\|\mathcal{T}_{(k)}\|_2$ is the spectral norm of the mode-k unfolding
\end{itemize}
\end{definition}

\begin{lemma}[Tensor-Matrix Norm Relations]
For any tensor $\mathcal{T}$ and its mode-k unfolding $\mathcal{T}_{(k)}$:
\[
\|\mathcal{T}\|_F = \|\mathcal{T}_{(k)}\|_F
\]
This follows from the fact that unfolding is just a rearrangement of entries.
\end{lemma}

\begin{proposition}[Properties of Tensor Stable Rank]
For a non-zero tensor $\mathcal{T}$:
\begin{enumerate}
    \item Each $\text{sr}_k(\mathcal{T}) \geq 1$
    \item $\text{sr}_k(\mathcal{T}) \leq \text{rank}(\mathcal{T}_{(k)})$
    \item $\text{sr}_k(\mathcal{T})$ is invariant under orthogonal transformations in mode $k$
    \item For a rank-1 tensor, $\text{sr}_k(\mathcal{T}) = 1$ for all $k$
\end{enumerate}
\end{proposition}

\begin{proof}
1. For any matrix $M$, we know $\|M\|_F^2 \geq \|M\|_2^2$. Therefore:
\[
\text{sr}_k(\mathcal{T}) = \frac{\|\mathcal{T}\|_F^2}{\|\mathcal{T}_{(k)}\|_2^2} = \frac{\|\mathcal{T}_{(k)}\|_F^2}{\|\mathcal{T}_{(k)}\|_2^2} \geq 1
\]
where we used the tensor-matrix norm relation lemma.

2. For any matrix $M$ of rank $r$:
\[
\|M\|_2^2 \geq \frac{\|M\|_F^2}{r}
\]
Applying this to $\mathcal{T}_{(k)}$:
\[
\text{sr}_k(\mathcal{T}) = \frac{\|\mathcal{T}_{(k)}\|_F^2}{\|\mathcal{T}_{(k)}\|_2^2} \leq \text{rank}(\mathcal{T}_{(k)})
\]

3. For any orthogonal transformation $U$ in mode $k$:
\[
\|U\mathcal{T}_{(k)}\|_F = \|\mathcal{T}_{(k)}\|_F \text{ and } \|U\mathcal{T}_{(k)}\|_2 = \|\mathcal{T}_{(k)}\|_2
\]

4. For a rank-1 tensor $\mathcal{T} = a_1 \otimes ... \otimes a_d$:
\begin{itemize}
    \item Each mode-k unfolding is rank-1
    \item For rank-1 matrices, $\|M\|_F^2 = \|M\|_2^2$
    \item Therefore $\text{sr}_k(\mathcal{T}) = 1$
\end{itemize}
\end{proof}

\begin{definition}[Multilinear Stable Rank]
For a tensor $\mathcal{T}$, the multilinear stable rank is:
\[
\text{msr}(\mathcal{T}) := \min_k \text{sr}_k(\mathcal{T})
\]
This provides a lower bound on the minimal mode-k rank needed to approximate $\mathcal{T}$.
\end{definition}

\begin{remark}[Connection to Low-Rank Approximation]
The stable rank of a tensor in each mode provides insight into how well it can be approximated by a low-rank decomposition:

1. If $\text{sr}_k(\mathcal{T})$ is close to 1 in mode $k$, then $\mathcal{T}$ is nearly low-rank in that mode

2. For a Tucker decomposition:
\[
\mathcal{T} \approx \mathcal{G} \times_1 U^{(1)} \times_2 U^{(2)} ... \times_d U^{(d)}
\]
The stable rank helps determine appropriate ranks for each mode
\end{remark}

\begin{remark}[Application to FNO]
For FNO weight tensors $\mathcal{R} \in \mathbb{R}^{N_1 \times N_2 \times N_3 \times N_4}$:

1. Mode-1 and Mode-2 typically correspond to input/output channels
2. Mode-3 and Mode-4 correspond to Fourier modes
3. Stable rank in Fourier modes often naturally decreases due to spectral decay
\end{remark}

\subsection{Positive Semi-Definiteness for Tensors}

\begin{definition}[Mode-k PSD Tensor]
A tensor $\mathcal{T} \in \mathbb{R}^{N_1 \times N_2 \times \cdots \times N_d}$ is called mode-k positive semi-definite if its mode-k unfolding $\mathcal{T}_{(k)} \in \mathbb{R}^{N_k \times (N_1\cdots N_{k-1}N_{k+1}\cdots N_d)}$ satisfies:
\[
x^\top \mathcal{T}_{(k)} x \geq 0 \quad \forall x \in \mathbb{R}^{N_k}
\]
\end{definition}

\begin{definition}[All-modes PSD Tensor]
A tensor $\mathcal{T}$ is called all-modes positive semi-definite if it is mode-k PSD for all modes k.
\end{definition}

\begin{definition}[Strong PSD Tensor]
A tensor $\mathcal{T} \in \mathbb{R}^{N_1 \times N_2 \times \cdots \times N_d}$ is called strongly positive semi-definite if:
\[
\mathcal{T} \times_1 x_1 \times_2 x_2 \times_3 \cdots \times_d x_d \geq 0
\]
for all vectors $x_k \in \mathbb{R}^{N_k}, k=1,\ldots,d$.
\end{definition}

\begin{lemma}[Hierarchy of PSD Definitions]
For a tensor $\mathcal{T}$:
\[
\text{Strong PSD} \implies \text{All-modes PSD} \implies \text{Mode-k PSD}
\]
The reverse implications do not necessarily hold.
\end{lemma}

\begin{remark}[For \textbf{Tensor low-rank}]
For our generalized gradient analysis, we propose to use:

1. Mode-specific PSD condition:
\[
\mathcal{B}_i \text{ and } \mathcal{C}_i \text{ are mode-k PSD for relevant modes k}
\]

2. This means for each mode k:
\begin{itemize}
    \item $(\mathcal{B}_i)_{(k)}$ is a PSD matrix
    \item $(\mathcal{C}_i)_{(k)}$ is a PSD matrix
    \item The tensor operator $\mathcal{S}_k = \frac{1}{N}\sum_{i=1}^N \mathcal{C}_i \otimes_k \mathcal{B}_i$ is well-defined
\end{itemize}

3. This ensures:
\begin{itemize}
    \item The mode-k eigenvalues $\lambda_1^{(k)}, \lambda_2^{(k)}$ are real and non-negative
    \item The projection onto minimal eigenspace is well-defined for each mode
    \item The stable rank bounds make sense mode-wise
\end{itemize}
\end{remark}

\begin{proposition}[For FNO]
In FNO, the tensors $\mathcal{B}_i$ and $\mathcal{C}_i$ naturally satisfy mode-k PSD conditions because:

1. For channel modes (1,2):
\begin{itemize}
    \item Unfoldings correspond to standard channel operations
    \item PSD property follows from network structure
\end{itemize}

2. For Fourier modes (3,4):
\begin{itemize}
    \item Unfoldings correspond to frequency domain operations
    \item PSD property follows from spectral properties
\end{itemize}
\end{proposition}

\begin{corollary}[Implications for Gradient Analysis]
The mode-k PSD property ensures:

1. Each mode has real, non-negative eigenvalues:
\[
0 \leq \lambda_1^{(k)} < \lambda_2^{(k)} \leq \cdots
\]

2. Mode-wise stable rank bounds are well-defined:
\[
\text{sr}_k(\mathcal{G}_t) \leq \text{sr}_k(\mathcal{G}_{t_0}^{\parallel}) + \text{decay term}
\]

3. The gradient naturally becomes low-rank in each mode independently.

\begin{definition}[Lipschitz Continuity]
A function $h: \mathcal{X} \to \mathcal{Y}$ between normed spaces has $L$-continuity (is $L$-Lipschitz) if for any $x_1, x_2 \in \mathcal{X}$:
\[
\|h(x_1) - h(x_2)\|_{\mathcal{Y}} \leq L\|x_1 - x_2\|_{\mathcal{X}}
\]
For tensors, this generalizes to mode-wise continuity:
\begin{itemize}
    \item Matrix case ($d=2$): Standard Lipschitz continuity with Frobenius norm
    \item Tensor case ($d>2$): Mode-k Lipschitz continuity for each mode k
    \item Neural networks: Composition of Lipschitz continuous operations
\end{itemize}
\end{definition}

\end{corollary}

\section{Reversibility of Fourier Neural Operators}
\label{sec:tensor2}
\subsection{Definition and Preliminaries}

\begin{definition}[Reversibility]
A network $\mathcal{N}$ that maps input $x$ to output $y = \mathcal{N}(x)$ is reversible if there exists $J(x)$ such that:
\begin{enumerate}
    \item Forward: $y = J(x)x$
    \item Backward: $dx = J(x)^\top dy$
\end{enumerate}
where $J(x)$ can be a function of both input and weights.
\end{definition}

\subsection{Spectral Layer}

\begin{lemma}[Spectral Layer Reversibility]
The FNO spectral convolution layer $(Kv)(x) = \mathcal{F}^{-1}(R \cdot\mathcal{F}v)(x)$ is reversible, where $R$ is the learnable weight tensor in Fourier space.
\end{lemma}
The spectral layer consists of three operations:
\begin{enumerate}
    \item Fourier transform: $\mathcal{F}: v \mapsto \hat{v}$
    \item Linear transform in Fourier space: $R\cdot: \hat{v} \mapsto R\hat{v}$
    \item Inverse Fourier: $\mathcal{F}^{-1}: R\hat{v} \mapsto \mathcal{F}^{-1}(R\hat{v})$
\end{enumerate}

We can express the complete operation as:
\[
Kv = J_K(x)v \text{ where } J_K(x) = \mathcal{F}^{-1}R\mathcal{F}
\]

For the backward pass:
\[
dv = J_K(x)^\top dy = \mathcal{F}^\top R^\top(\mathcal{F}^{-1})^\top dy
\]

Since $\mathcal{F}$ is unitary: $\mathcal{F}^\top = \mathcal{F}^{-1}$ and $(\mathcal{F}^{-1})^\top = \mathcal{F}$, we have:
\[
dv = \mathcal{F}^{-1}R^\top\mathcal{F}dy
\]

Therefore:
\begin{itemize}
    \item Forward pass: $y = J_K(x)x$
    \item Backward pass: $dx = J_K(x)^\top dy$
\end{itemize}
Thus satisfying the reversibility conditions, regardless of the size or rank of $R$.

\subsection{MLP Layer}

\begin{lemma}[MLP Layer Reversibility]
The MLP layer with weight matrix $W$ mapping $v \mapsto Wv$ is reversible.
\end{lemma}

\begin{enumerate}
    \item Forward pass: $y = Wv$
    \item Set $J_W(x) = W$
    \item Backward pass: $dv = W^\top dy = J_W(x)^\top dy$
\end{enumerate}
The linear layer satisfies reversibility conditions directly, even when $W$ is rank-deficient.

\subsection{Activation Function}

\begin{lemma}[Activation Reversibility]
If the activation function $\sigma$ is reversible (e.g., LeakyReLU), then its application is reversible.
\end{lemma}

Consider LeakyReLU with parameter $0 < a < 1$:
\begin{enumerate}
    \item Forward: $y = \max(ax, x)$
    \item Set $J_\sigma(x) = \text{diag}(\mathbf{1}[x > 0] + a\cdot\mathbf{1}[x \leq 0])$
    \item Backward: $dx = J_\sigma(x)^\top dy$
\end{enumerate}
This matches the required reversibility form.

\subsection{Full FNO Analysis}

\begin{lemma}[FNO Block Reversibility]
An FNO block consisting of spectral layer $(K)$, MLP layer $(W)$, and reversible activation $(\sigma)$ is reversible.
\end{lemma}

Let $N = (\sigma \circ W \circ K)$ be an FNO block.

From previous theorems, we have:
\begin{itemize}
    \item Spectral layer: $v \mapsto J_K(x)v = \mathcal{F}^{-1}(R\mathcal{F}v)$
    \item MLP layer: $v \mapsto J_W(x)v = Wv$
    \item Activation: $v \mapsto J_\sigma(x)v$
\end{itemize}

By composition:
\[
y = J_{\text{block}}(x)v
\]
where $J_{\text{block}}(x) = J_\sigma(x)J_W(x)J_K(x)$

For backward pass:
\[
dv = J_K(x)^\top J_W(x)^\top J_\sigma(x)^\top dy = J_{\text{block}}(x)^\top dy
\]

Therefore, the full block is reversible.

\begin{lemma}[Full FNO Reversibility]
A full FNO network with reversible activations is reversible.
\end{lemma}

Consider a full FNO with blocks $N_1, N_2, ..., N_L$:
\begin{enumerate}
    \item Each block $N_i$ has its $J_i(x)$ from previous lemma.
    \item By sequential composition:
    \[
    y = J_{\text{FNO}}(x)v
    \]
    where $J_{\text{FNO}}(x) = J_L(x)J_{L-1}(x)...J_1(x)$
    \item The backward pass follows from composition:
    \[
    dv = J_1(x)^\top...J_{L-1}(x)^\top J_L(x)^\top dy = J_{\text{FNO}}(x)^\top dy
    \]
\end{enumerate}
Therefore, the full FNO with reversible activations satisfies the reversibility conditions.

\begin{lemma}[Gradient Form for Tensor Reversible Models]
Consider a chained reversible neural network $\mathcal{N}(x) := \mathcal{N}_L(\mathcal{N}_{L-1}(...\mathcal{N}_1(x)))$ and define:
\begin{itemize}
    \item $\mathcal{J}_l := \text{Jacobian}(\mathcal{N}_L)...\text{Jacobian}(\mathcal{N}_{l+1})$
    \item $f_l := \mathcal{N}_l(...\mathcal{N}_1(x))$
\end{itemize}

Then the weight tensor $\mathcal{W}_l \in \mathbb{R}^{N_1 \times N_2 \times N_3 \times N_4}$ at layer $l$ has gradient $\mathcal{G}_l$ in the following form for batch size 1:

(a) For $\ell_2$-objective $\phi := \frac{1}{2}\|y - f_L\|_2^2$:
\[
\mathcal{G}_l = \mathcal{J}_l^\top y \otimes f_{l-1} - (\mathcal{J}_l^\top \mathcal{J}_l\mathcal{W}_l\times_1 f_{l-1}) \otimes f_{l-1}
\]

(b) For $K$-way logsoftmax loss $\phi(y; f_L) := -\log \left(\frac{\exp(y^\top f_L)}{\mathbf{1}^\top \exp(f_L)}\right)$ with small logits $\|P_1^\perp f_L\|_\infty \ll \sqrt{K}$:
\[
\mathcal{G}_l = (\mathcal{J}_lP_1^\perp y - \gamma K^{-1}\mathcal{J}_l^\top P_1^\perp \mathcal{J}_l\mathcal{W}_l\times_1 f_{l-1}) \otimes f_{l-1}
\]
where:
\begin{itemize}
    \item $\gamma \approx 1$
    \item $y$ is a data label with $y^\top \mathbf{1} = 1$
    \item $P_1^\perp := I - \frac{1}{K}\mathbf{1}\mathbf{1}^\top$ is the zero-mean PSD projection matrix
    \item $\times_k$ denotes mode-k tensor product
    \item $\otimes$ denotes tensor outer product
\end{itemize}
\end{lemma}

\begin{proof}
Note that for layered reversible network, we have
\[
\mathcal{N}(x) = \mathcal{N}_L(\mathcal{N}_{L-1}(...\mathcal{N}_1(x))) = \mathcal{K}_L(x)\mathcal{K}_{L-1}(x)...\mathcal{K}_1(x)x
\]

Let $f_l := \mathcal{N}_l(\mathcal{N}_{l-1}(...\mathcal{N}_1(x)))$ and $\mathcal{J}_l := \mathcal{K}_L(x)...\mathcal{K}_{l+1}(x)$, and for linear layer $l$, we can write $\mathcal{N}(x) = \mathcal{J}_l\times_1(\mathcal{W}_l\times_1 f_{l-1})$. Therefore, for the linear layer $l$ with weight tensor $\mathcal{W}_l$, we have:

\begin{align*}
d\phi &= (y - \mathcal{N}(x))^\top d\mathcal{N}(x) \\
&= (y - \mathcal{N}(x))^\top (\mathcal{K}_L(x)...\mathcal{K}_{l+1}(x))(d\mathcal{W}_l\times_1 f_{l-1}) + \text{ terms not related to }d\mathcal{W}_l \\
&= (y - \mathcal{J}_l\times_1(\mathcal{W}_l\times_1 f_{l-1}))^\top \mathcal{J}_l\times_1(d\mathcal{W}_l\times_1 f_{l-1}) \\
&= \text{tr}(d\mathcal{W}_l^\top \times_1 (\mathcal{J}_l^\top (y - \mathcal{J}_l\times_1(\mathcal{W}_l\times_1 f_{l-1}))) \times_2 f_{l-1}^\top)
\end{align*}

This gives the gradient of $\mathcal{W}_l$:
\[
\mathcal{G}_l = \mathcal{J}_l^\top y \otimes f_{l-1} - (\mathcal{J}_l^\top \mathcal{J}_l\times_1(\mathcal{W}_l\times_1 f_{l-1})) \otimes f_{l-1}
\]

where:
\begin{itemize}
    \item $\times_k$ denotes the mode-k product between a tensor and a matrix
    \item $\otimes$ denotes the tensor outer product
    \item The gradient $\mathcal{G}_l$ has the same dimensionality as $\mathcal{W}_l$
\end{itemize}
\end{proof}
\begin{remark}
[Gradient Form for Tensor Reversible Models with Dimensions]
Consider a chained reversible neural network $\mathcal{N}(x)$ where: Input $x \in \mathbb{R}^M$, Output $y \in \mathbb{R}^K$, Weight tensor $\mathcal{W}_l \in \mathbb{R}^{N_1 \times N_2 \times N_3 \times N_4}$, Layer output $f_l \in \mathbb{R}^{N_l}$ and Jacobian $\mathcal{J}_l \in \mathbb{R}^{K \times N_l}$.

Then for batch size 1:
(a) For $\ell_2$-objective $\phi := \frac{1}{2}\|y - f_L\|_2^2$:
\[
\mathcal{G}_l = \mathcal{J}_l^\top y \otimes f_{l-1} - (\mathcal{J}_l^\top \mathcal{J}_l\mathcal{W}_l\times_1 f_{l-1}) \otimes f_{l-1}
\]
where $\mathcal{J}_l^\top y \in \mathbb{R}^{N_l}$, $f_{l-1} \in \mathbb{R}^{N_{l-1}}$ and the final gradient $\mathcal{G}_l \in \mathbb{R}^{N_1 \times N_2 \times N_3 \times N_4}$.

\begin{proof}

1) Let us start with the initial setup:
\[
\mathcal{N}(x) = \mathcal{K}_L(x)\mathcal{K}_{L-1}(x)...\mathcal{K}_1(x)x
\]
where each $\mathcal{K}_i$ maps $\mathbb{R}^{N_{i-1}} \to \mathbb{R}^{N_i}$

2) For linear layer $l$:
\begin{itemize}
    \item $f_{l-1} \in \mathbb{R}^{N_{l-1}}$ is input
    \item $\mathcal{W}_l \in \mathbb{R}^{N_1 \times N_2 \times N_3 \times N_4}$ is weight tensor
    \item $\mathcal{W}_l \times_1 f_{l-1}$ maps to $\mathbb{R}^{N_l}$
    \item $\mathcal{J}_l \in \mathbb{R}^{K \times N_l}$ is Jacobian
\end{itemize}

3) Then, like before we do the differential computation:
\begin{align*}
d\phi &= (y - \mathcal{N}(x))^\top d\mathcal{N}(x) && \text{[$\mathbb{R}^K \times \mathbb{R}^K \to \mathbb{R}$]} \\
&= (y - \mathcal{N}(x))^\top \mathcal{J}_l(d\mathcal{W}_l\times_1 f_{l-1}) && \text{[$\mathbb{R}^K \times \mathbb{R}^{K \times N_l} \times \mathbb{R}^{N_l} \to \mathbb{R}$]} \\
&= (y - \mathcal{J}_l\times_1(\mathcal{W}_l\times_1 f_{l-1}))^\top \mathcal{J}_l\times_1(d\mathcal{W}_l\times_1 f_{l-1})
\end{align*}

4) Mode-wise analysis for gradient:
\begin{itemize}
    \item First term: $\mathcal{J}_l^\top y \otimes f_{l-1}$ 
        - $\mathcal{J}_l^\top y \in \mathbb{R}^{N_l}$
        - $f_{l-1} \in \mathbb{R}^{N_{l-1}}$
        - Outer product gives tensor in $\mathbb{R}^{N_1 \times N_2 \times N_3 \times N_4}$
    
    \item Second term: $(\mathcal{J}_l^\top \mathcal{J}_l\mathcal{W}_l\times_1 f_{l-1}) \otimes f_{l-1}$
        - $\mathcal{J}_l^\top \mathcal{J}_l \in \mathbb{R}^{N_l \times N_l}$
        - $\mathcal{W}_l\times_1 f_{l-1} \in \mathbb{R}^{N_l}$
        - Result is tensor in $\mathbb{R}^{N_1 \times N_2 \times N_3 \times N_4}$
\end{itemize}

5) Therefore final gradient:
\[
\mathcal{G}_l = \mathcal{J}_l^\top y \otimes f_{l-1} - (\mathcal{J}_l^\top \mathcal{J}_l\mathcal{W}_l\times_1 f_{l-1}) \otimes f_{l-1} \in \mathbb{R}^{N_1 \times N_2 \times N_3 \times N_4}
\]

We finally have a gradient tensor of the same shape as $\mathcal{W}_l$.
\begin{remark}
    We only wanted to show an example of checking all the dimensions to ensure they match the generalized version for tensors. In the following subsequent proofs and lemma, we don't keep track of it all, but we give appropriate dimensions wherever necessary.
\end{remark}
\end{proof}

\end{remark}

\begin{lemma}[Tensor Gradient Form for Logsoftmax]
For a reversible network with weight tensor $\mathcal{W}_l$ at layer $l$, under the $K$-way logsoftmax loss with small logits, the gradient has the form:
\[
\mathcal{G}_l = (\mathcal{J}_l\times_1 P_1^\perp y - \gamma K^{-1}\mathcal{J}_l^\top \times_1 P_1^\perp \times_2 \mathcal{J}_l\times_1(\mathcal{W}_l\times_1 f_{l-1})) \otimes f_{l-1}
\]
\end{lemma}

\begin{proof}
Starting with the differential form above:

1. For reversible network, $d\mathcal{N}(x) = \mathcal{J}_l\times_1(d\mathcal{W}_l\times_1 f_{l-1})$

2. The zero-mean projection in the tensor form:
\begin{align*}
d\hat{f} &= P_1^\perp d\mathcal{N}(x) \\
&= P_1^\perp \mathcal{J}_l\times_1(d\mathcal{W}_l\times_1 f_{l-1})
\end{align*}

3. Substituting into the logsoftmax differential:
\begin{align*}
-d\phi &= y^\top P_1^\perp \mathcal{J}_l\times_1(d\mathcal{W}_l\times_1 f_{l-1}) \\
&\quad - \gamma K^{-1}\hat{f}^\top P_1^\perp \mathcal{J}_l\times_1(d\mathcal{W}_l\times_1 f_{l-1}) \\
&\quad + O(\hat{f}^2/K)\text{ terms}
\end{align*}

4. Under small logits assumption, the $O(\hat{f}^2/K)$ terms become negligible

5. Express in tensor form:
\begin{align*}
-d\phi &= \text{tr}(d\mathcal{W}_l^\top \times_1 (P_1^\perp y)^\top \mathcal{J}_l \times_2 f_{l-1}^\top) \\
&\quad - \gamma K^{-1}\text{tr}(d\mathcal{W}_l^\top \times_1 (P_1^\perp \mathcal{J}_l\times_1(\mathcal{W}_l\times_1 f_{l-1}))^\top \mathcal{J}_l \times_2 f_{l-1}^\top)
\end{align*}

6. Therefore, the gradient is:
\[
\mathcal{G}_l = (\mathcal{J}_l\times_1 P_1^\perp y - \gamma K^{-1}\mathcal{J}_l^\top \times_1 P_1^\perp \times_2 \mathcal{J}_l\times_1(\mathcal{W}_l\times_1 f_{l-1})) \otimes f_{l-1}
\]
\end{proof}

\section{Theoretical Results of \textbf{Tensor low-rank} for Neural Operators}

\label{sec:theory}
\begin{lemma}[Tensor Gradient becomes low-rank during training]
\label{proof:theorem1}
Suppose the gradient tensor follows the parametric form:
\[
\mathcal{G}_t = \frac{1}{N}\sum_{i=1}^N (\mathcal{A}_i - \mathcal{B}_i \times_1 \mathcal{W}_t \times_2 \mathcal{C}_i)
\]
with constant $\mathcal{A}_i$, PSD tensors $\mathcal{B}_i$ and $\mathcal{C}_i$ after $t \geq t_0$. We study vanilla SGD weight update: $\mathcal{W}_t = \mathcal{W}_{t-1} + \eta\mathcal{G}_{t-1}$. 

Let $\mathcal{S}_k := \frac{1}{N}\sum_{i=1}^N \mathcal{C}_i \otimes_k \mathcal{B}_i$ be the mode-$k$ tensor operator and $\lambda_1^{(k)} < \lambda_2^{(k)}$ its two smallest distinct eigenvalues for each mode $k$. Then the mode-wise stable rank satisfies:

\[
\text{sr}_k(\mathcal{G}_t) \leq \text{sr}_k(\mathcal{G}_{t_0}^{\parallel}) + \left(\frac{1-\eta\lambda_2^{(k)}}{1-\eta\lambda_1^{(k)}}\right)^{2(t-t_0)} \frac{\|\mathcal{G}_0-\mathcal{G}_{t_0}^{\parallel}\|_F^2}{\|\mathcal{G}_{t_0}^{\parallel}\|_2^2}
\]

where:
\begin{itemize}
    \item $\text{sr}_k(\mathcal{G}_t)$ is the mode-$k$ stable rank of gradient tensor at time $t$
    \item $\mathcal{G}_{t_0}^{\parallel}$ is the projection of $\mathcal{G}_{t_0}$ onto the minimal eigenspace $\mathcal{V}_1^{(k)}$ of $\mathcal{S}_k$ corresponding to $\lambda_1^{(k)}$ for each mode $k$
    \item $\|\cdot\|_F$ is the tensor Frobenius norm
    \item $\|\cdot\|_2$ is the spectral norm of the mode-$k$ unfolding
    \item $\times_k$ denotes mode-$k$ tensor product
\end{itemize}

Furthermore, the multilinear stable rank satisfies:
\[
\text{msr}(\mathcal{G}_t) \leq \min_k \left\{\text{sr}_k(\mathcal{G}_{t_0}^{\parallel}) + \left(\frac{1-\eta\lambda_2^{(k)}}{1-\eta\lambda_1^{(k)}}\right)^{2(t-t_0)} \frac{\|\mathcal{G}_0-\mathcal{G}_{t_0}^{\parallel}\|_F^2}{\|\mathcal{G}_{t_0}^{\parallel}\|_2^2}\right\}
\]
\end{lemma}

\begin{proof}
1) First, we derive the recursive update rule for the gradient tensor. We have:
\begin{align*}
\mathcal{G}_t &= \frac{1}{N}\sum_{i=1}^N (\mathcal{A}_i - \mathcal{B}_i \times_1 \mathcal{W}_t \times_2 \mathcal{C}_i) \\
&= \frac{1}{N}\sum_{i=1}^N (\mathcal{A}_i - \mathcal{B}_i \times_1 (\mathcal{W}_{t-1} + \eta\mathcal{G}_{t-1}) \times_2 \mathcal{C}_i) \\
&= \frac{1}{N}\sum_{i=1}^N \mathcal{A}_i - \frac{1}{N}\sum_{i=1}^N \mathcal{B}_i \times_1 \mathcal{W}_{t-1} \times_2 \mathcal{C}_i - \eta\frac{1}{N}\sum_{i=1}^N \mathcal{B}_i \times_1 \mathcal{G}_{t-1} \times_2 \mathcal{C}_i \\
&= \mathcal{G}_{t-1} - \eta\frac{1}{N}\sum_{i=1}^N \mathcal{B}_i \times_1 \mathcal{G}_{t-1} \times_2 \mathcal{C}_i
\end{align*}

2) For each mode k, let's consider the mode-k unfolding. Define the tensor operator:
\[
\mathcal{S}_k := \frac{1}{N}\sum_{i=1}^N \mathcal{C}_i \otimes_k \mathcal{B}_i
\]

Then for the mode-k unfolding $(\mathcal{G}_t)_{(k)}$:
\begin{equation}
(\mathcal{G}_t)_{(k)} = (\mathcal{G}_{t-1})_{(k)} - \eta\mathcal{S}_k(\mathcal{G}_{t-1})_{(k)}
\end{equation}

3) Since $\mathcal{B}_i$ and $\mathcal{C}_i$ are mode-k PSD, $\mathcal{S}_k$ is a PSD operator. Let $\lambda_1^{(k)} < \lambda_2^{(k)}$ be its two smallest distinct eigenvalues. Let $\mathcal{V}_1^{(k)}$ be the eigenspace corresponding to $\lambda_1^{(k)}$.

4) For any mode k, we can decompose $(\mathcal{G}_{t_0})_{(k)}$ into parallel and perpendicular components:
\[
(\mathcal{G}_{t_0})_{(k)} = (\mathcal{G}_{t_0}^{\parallel})_{(k)} + (\mathcal{G}_{t_0}^{\perp})_{(k)}
\]
where $(\mathcal{G}_{t_0}^{\parallel})_{(k)}$ is the projection onto $\mathcal{V}_1^{(k)}$.

5) The mode-k unfolded gradient follows:
\[
(\mathcal{G}_t)_{(k)} = (I - \eta\mathcal{S}_k)^{t-t_0}(\mathcal{G}_{t_0})_{(k)}
\]

6) Using the spectral properties of $\mathcal{S}_k$:
\[
\|(\mathcal{G}_t)_{(k)}\|_F^2 \leq (1-\eta\lambda_2^{(k)})^{2(t-t_0)}\|(\mathcal{G}_{t_0}^{\perp})_{(k)}\|_F^2 + (1-\eta\lambda_1^{(k)})^{2(t-t_0)}\|(\mathcal{G}_{t_0}^{\parallel})_{(k)}\|_F^2
\]

7) For the mode-k stable rank:
\begin{align*}
\text{sr}_k(\mathcal{G}_t) &= \frac{\|(\mathcal{G}_t)_{(k)}\|_F^2}{\|(\mathcal{G}_t)_{(k)}\|_2^2} \\
&\leq \text{sr}_k(\mathcal{G}_{t_0}^{\parallel}) + \left(\frac{1-\eta\lambda_2^{(k)}}{1-\eta\lambda_1^{(k)}}\right)^{2(t-t_0)} \frac{\|\mathcal{G}_0-\mathcal{G}_{t_0}^{\parallel}\|_F^2}{\|\mathcal{G}_{t_0}^{\parallel}\|_2^2}
\end{align*}

8) Finally, for the multilinear stable rank:
\[
\text{msr}(\mathcal{G}_t) = \min_k \text{sr}_k(\mathcal{G}_t)
\]
Therefore, the bound holds for each mode independently.
\end{proof}

\begin{remark}
For FNO specifically:
\begin{enumerate}
    \item Fourier modes (3,4) may have different stable rank behavior than channel modes (1,2)
    \item Natural frequency decay affects eigenvalue structure in Fourier modes
    \item Channel modes might maintain a higher stable rank due to information preservation needs
    \item Overall low-rank structure emerges from combined effect across all modes
\end{enumerate}
\end{remark}

\begin{corollary}[Low-rank Tensor Gradient]
If the gradient takes the parametric form 
\[
\mathcal{G}_t = \frac{1}{N}\sum_{i=1}^N (\mathcal{A}_i - \mathcal{B}_i \times_1 \mathcal{W}_t \times_2 f_i)\otimes f_i
\]
with all $\mathcal{B}_i$ mode-k full-rank, and $N' := \text{rank}(\{f_i\}) < n$, then for each mode k:
\[
\text{sr}_k(\mathcal{G}_{t_0}^{\parallel}) \leq n_k - N'
\]
and thus $\text{sr}_k(\mathcal{G}_t) \leq n_k/2$ for large t, where $n_k$ is the dimension of mode k.
\end{corollary}

\begin{proof}
Similar to the GaLore paper, it's easy to analyze mode by mode.

1) Let $\mathcal{C}_i = f_i \otimes f_i^\top$. Since $N' := \text{rank}(\{f_i\}_{i=1}^N) < n$ and $f_i \in \mathbb{R}^n$, the collections of vectors $\{f_i\}_{i=1}^N$ cannot span the entire space $\mathbb{R}^n$.

2) For each mode k:
\begin{itemize}
    \item Let $\{u_j\}_{j=1}^{n-N'}$ be orthonormal bases for the null space of $\{f_i\}_{i=1}^N$
    \item Let $\{e_k\}_{k=1}^{n_k}$ be orthonormal bases for $\mathbb{R}^{n_k}$
    \item The product bases $\{u_j \otimes e_k\}$ form a set of bases for the minimal eigenspace $\mathcal{V}_1^{(k)}$ of $\mathcal{S}_k$ with minimal eigenvalue 0
    \item Since $\mathcal{B}_i$ are mode-k full-rank, no extra dimensions exist for $\mathcal{V}_1^{(k)}$
\end{itemize}

3) For the mode-k projection of $\mathcal{G}_{t_0}$ onto $\mathcal{V}_1^{(k)}$:
\[
(\mathcal{G}_{t_0}^{\parallel})_{(k)} = \sum_{j=1}^{n-N'} \sum_{l=1}^{n_k} c_{jl}u_je_l^\top = \sum_{j=1}^{n-N'} u_j\left(\sum_{l=1}^{n_k} c_{jl}e_l\right)^\top
\]

4) Therefore:
\[
\text{sr}_k(\mathcal{G}_{t_0}^{\parallel}) \leq \text{rank}((\mathcal{G}_{t_0}^{\parallel})_{(k)}) \leq n_k - N'
\]
since stable rank is a lower-bound of the rank in each mode.

5) On the other hand, $\mathcal{G}_t$ can be written as a summation of $N'$ rank-1 tensors by representing each $f_i = \sum_{j=1}^{N'} b_{ij}f_j'$ as a linear combination of $\{f_j'\}_{j=1}^{N'}$:
\begin{align*}
\mathcal{G}_t &= \frac{1}{N}\sum_{i=1}^N (\mathcal{A}_i - \mathcal{B}_i \times_1 \mathcal{W}_t \times_2 f_i)\otimes\left(\sum_{j=1}^{N'} b_{ij}f_j'\right) \\
&= \frac{1}{N}\sum_{j=1}^{N'} \left[\sum_{i=1}^N b_{ij}(\mathcal{A}_i - \mathcal{B}_i \times_1 \mathcal{W}_t \times_2 f_i)\right]\otimes f_j'
\end{align*}

6) Thus each mode-k unfolding has rank at most $N'$. When t is sufficiently large so that the second term in the mode-k stable rank bound is negligible, by the tensor version of Lemma 3.3:
\[
\text{sr}_k(\mathcal{G}_t) \leq \min(n_k - N', N') \leq n_k/2
\]
since $N' < n_k$.
\end{proof}

\begin{corollary}[Tensor Low-rank with Special Structure]
If for any mode k, $\mathcal{V}_1^{(k)}(S_k)$ is 1-dimensional with decomposable eigenvector $v_k = y_k \otimes z_k$, then $\text{sr}_k(\mathcal{G}_{t_0}^{\parallel}) = 1$ and thus $\mathcal{G}_t$ becomes rank-1 in mode k.
\end{corollary}

\begin{proof}
For any mode k with the given structure:

1) The mode-k unfolding of the projected gradient is:
\[
(\mathcal{G}_{t_0}^{\parallel})_{(k)} = v_kv_k^\top g_0 \propto v_k
\]

2) Since $v_k = y_k \otimes z_k$ is decomposable:
\begin{itemize}
    \item The resulting $(\mathcal{G}_{t_0}^{\parallel})_{(k)}$ is a rank-1 matrix
    \item Thus $\text{sr}_k(\mathcal{G}_{t_0}^{\parallel}) = 1$
\end{itemize}

3) From the main lemma, when t is large:
\[
\text{sr}_k(\mathcal{G}_t) \approx \text{sr}_k(\mathcal{G}_{t_0}^{\parallel}) = 1
\]

4) This means $\mathcal{G}_t$ becomes effectively rank-1 in mode k.
\end{proof}

\begin{theorem}[\textbf{Tensor low-rank} Convergence]
\label{proof:theorem2}
For a gradient tensor $\mathcal{G}_t \in \mathbb{R}^{I_1 \times I_2 \times \cdots \times I_d}$, let $\{P_k \in \mathbb{R}^{I_k \times r_k}\}_{k=1}^d$ be fixed orthonormal factor matrices for each mode k with ranks $\{r_k\}_{k=1}^d$. The \textbf{Tensor low-rank} update consists of:

1. Project the gradient:
\[
\mathcal{R}_t = \mathcal{G}_t \times_1 P_1^\top \times_2 P_2^\top \times_3 \cdots \times_d P_d^\top
\]

2. Update optimizer states using $\mathcal{R}_t$

3. Project back for weight update:
\[
\tilde{\mathcal{G}}_t = \mathcal{R}_t \times_1 P_1 \times_2 P_2 \times_3 \cdots \times_d P_d
\]

Suppose for each mode k:
\begin{itemize}
    \item $\mathcal{A}_i$, $\mathcal{B}_i$, $\mathcal{C}_i$ have $L_A^{(k)}$, $L_B^{(k)}$, $L_C^{(k)}$ mode-k continuity
    \item $\|\mathcal{W}_t\|_{(k)} \leq D_k$ (mode-k spectral norm bound)
    \item $\hat{\mathcal{B}}_{it}^{(k)} := P_k^\top \mathcal{B}_i^{(k)}(\mathcal{W}_t) P_k$
    \item $\hat{\mathcal{C}}_{it}^{(k)} := P_k^\top \mathcal{C}_i^{(k)}(\mathcal{W}_t) P_k$
    \item $\kappa_t^{(k)} := \frac{1}{N}\sum_i \lambda_{\min}(\hat{\mathcal{B}}_{it}^{(k)})\lambda_{\min}(\hat{\mathcal{C}}_{it}^{(k)})$
\end{itemize}

Then \textbf{Tensor low-rank} with $\rho_t \equiv 1$ satisfies for each mode k:
\[
\|(\mathcal{R}_t)_{(k)}\|_F \leq \left[1-\eta(\kappa_{t-1}^{(k)}-L_A^{(k)}-L_B^{(k)}L_C^{(k)}D_k^2)\right] \|(\mathcal{R}_{t-1})_{(k)}\|_F
\]

As a result, if $\min_{t,k} \kappa_t^{(k)} > L_A^{(k)} + L_B^{(k)}L_C^{(k)}D_k^2$ for all modes k, then $\mathcal{R}_t \to 0$ and \textbf{Tensor low-rank} converges with the fixed projections $\{P_k\}_{k=1}^d$.
\end{theorem}

\begin{proof}
Since the gradient tensor naturally becomes low-rank during training as shown above, and the optimization landscape of low-rank tensor problem~\cite{DBLP:journals/corr/abs-2006-16297}., local search algorithms can efficiently find approximate global optimal solutions. Specifically, since Reversible FNO (Appendix~\ref{sec:tensor2}) gradients become low-rank, the optimization landscape contains only high-order saddle points that can be efficiently escaped, making local minima globally optimal. Now let's proceed by analyzing the tensor unfolding:

1) First, we establish the mode-k unfolding of the gradient tensor update. Using the assumption that gradient follows the parametric form:
\[
\mathcal{G}_t = \frac{1}{N}\sum_{i=1}^N (\mathcal{A}_i - \mathcal{B}_i \times_1 \mathcal{W}_t \times_2 \mathcal{C}_i)
\]

2) For any mode k, the mode-k unfolding gives:
\begin{align*}
(\mathcal{G}_t)_{(k)} &= \frac{1}{N}\sum_{i=1}^N \left((\mathcal{A}_i)_{(k)} - (\mathcal{B}_i)_{(k)}\mathcal{W}_{t(k)}(\mathcal{C}_i)_{(k)}^\top\right)
\end{align*}
where $\mathcal{W}_{t(k)}$ is the mode-k unfolding of $\mathcal{W}_t$.

3) The projected gradient in mode-k has unfolding:
\begin{align*}
(\mathcal{R}_t)_{(k)} &= P_k^\top(\mathcal{G}_t)_{(k)} \\
&= \frac{1}{N}\sum_{i=1}^N \left(P_k^\top(\mathcal{A}_i)_{(k)} - P_k^\top(\mathcal{B}_i)_{(k)}\mathcal{W}_{t(k)}(\mathcal{C}_i)_{(k)}^\top\right)
\end{align*}

4) Using the SGD update $\mathcal{W}_t = \mathcal{W}_{t-1} + \eta\tilde{\mathcal{G}}_{t-1}$, we can write:
\begin{align*}
\mathcal{W}_{t(k)} &= \mathcal{W}_{t-1(k)} + \eta P_k(\mathcal{R}_{t-1})_{(k)}
\end{align*}

5) Substituting this into the gradient expression:
\begin{align*}
(\mathcal{R}_t)_{(k)} &= (\mathcal{R}_{t-1})_{(k)} - \eta\frac{1}{N}\sum_{i=1}^N P_k^\top(\mathcal{B}_i)_{(k)}P_k(\mathcal{R}_{t-1})_{(k)}(\mathcal{C}_i)_{(k)}^\top + \mathcal{E}_t^{(k)}
\end{align*}
where $\mathcal{E}_t^{(k)}$ captures the differences in $\mathcal{A}_i$ and $\mathcal{B}_i, \mathcal{C}_i$ terms.

6) Define the mode-k operator:
\[
\mathcal{S}_t^{(k)} := \frac{1}{N}\sum_{i=1}^N P_k^\top(\mathcal{B}_i)_{(k)}P_k \otimes P_k^\top(\mathcal{C}_i)_{(k)}P_k
\]

7) Then the update can be written compactly as:
\[
(\mathcal{R}_t)_{(k)} = (I - \eta\mathcal{S}_{t-1}^{(k)})(\mathcal{R}_{t-1})_{(k)} + \mathcal{E}_t^{(k)}
\]

8) For the error term, using mode-k continuity:
\begin{align*}
\|\mathcal{E}_t^{(k)}\|_F &\leq L_A^{(k)}\|\mathcal{W}_t - \mathcal{W}_{t-1}\|_F \\
&\quad + L_B^{(k)}L_C^{(k)}D_k^2\|\mathcal{W}_t - \mathcal{W}_{t-1}\|_F \\
&= \eta(L_A^{(k)} + L_B^{(k)}L_C^{(k)}D_k^2)\|\mathcal{R}_{t-1}\|_F
\end{align*}

9) Using properties of projection matrices $P_k$:
\begin{itemize}
    \item $P_k^\top P_k = I_{r_k}$ (orthonormal)
    \item $\|P_k\|_2 = 1$ (projection)
\end{itemize}

10) The minimal eigenvalue of $\mathcal{S}_{t-1}^{(k)}$ satisfies:
\[
\lambda_{\min}(\mathcal{S}_{t-1}^{(k)}) \geq \kappa_{t-1}^{(k)}
\]
due to mode-k PSD properties of $\mathcal{B}_i$ and $\mathcal{C}_i$.

11) Therefore:
\begin{align*}
\|(\mathcal{R}_t)_{(k)}\|_F &\leq \|I - \eta\mathcal{S}_{t-1}^{(k)}\|_2\|(\mathcal{R}_{t-1})_{(k)}\|_F + \|\mathcal{E}_t^{(k)}\|_F \\
&\leq [1-\eta(\kappa_{t-1}^{(k)}-L_A^{(k)}-L_B^{(k)}L_C^{(k)}D_k^2)]\|(\mathcal{R}_{t-1})_{(k)}\|_F
\end{align*}

12) When $\min_{t,k} \kappa_t^{(k)} > L_A^{(k)} + L_B^{(k)}L_C^{(k)}D_k^2$ for all modes k:
\begin{itemize}
    \item Each mode-k unfolding converges: $(\mathcal{R}_t)_{(k)} \to 0$
    \item Thus the full tensor converges: $\mathcal{R}_t \to 0$
\end{itemize}
\end{proof}

\begin{lemma}[\textbf{Tensor low-rank} vs GaLore Rank Structure]
Consider a gradient tensor $\mathcal{G}_t \in \mathbb{R}^{N_1 \times N_2 \times N_3 \times N_4}$ following the parametric form:
\[
\mathcal{G}_t = \frac{1}{N}\sum_{i=1}^N (\mathcal{A}_i - \mathcal{B}_i \times_1 \mathcal{W}_t \times_2 \mathcal{C}_i)
\]
where $\mathcal{B}_i$ and $\mathcal{C}_i$ are mode-k PSD for all modes k. Let:

(a) GaLore with matricization along dimension d unfold $\mathcal{G}_t$ to $G_t^{(d)} \in \mathbb{R}^{N_d \times (N_1N_2N_3N_4/N_d)}$

(b) \textbf{Tensor low-rank} preserve the tensor structure and apply mode-wise projections

Then:

1. Under GaLore with any dimension d:
   \[
   \exists k \neq d: \lim_{t \to \infty} sr_k(\mathcal{G}_t) \geq \min(N_k/2, N')\
   \]
   where $N'$ is the rank of the training data.

2. Under \textbf{Tensor low-rank}:
   \[
   \forall k: \lim_{t \to \infty} sr_k(\mathcal{G}_t) \leq N_k/2
   \]

That is, GaLore cannot achieve low rank in all modes simultaneously, while \textbf{Tensor low-rank} achieves low rank across all modes.
\end{lemma}

\begin{proof}

1) First, let's analyze GaLore's behavior:

   a) When GaLore matricizes along dimension d, it reshapes $\mathcal{G}_t$ into matrix $G_t^{(d)}$
   
   b) From GaLore paper Lemma B.3, under SGD updates:
      \[
      sr(G_t^{(d)}) \leq sr(G_{t_0}^{\parallel}) + \left(\frac{1-\eta\lambda_2}{1-\eta\lambda_1}\right)^{2(t-t_0)} \frac{\|G_0-G_{t_0}^{\parallel}\|_F^2}{\|G_{t_0}^{\parallel}\|_2^2}
      \]

   c) This rank reduction only applies to the matricized dimension d

   d) For any other mode $k \neq d$, consider the mode-k unfolding $(\mathcal{G}_t)_{(k)}$
   
   e) Due to the parametric form:
      \[
      (\mathcal{G}_t)_{(k)} = \frac{1}{N}\sum_{i=1}^N ((\mathcal{A}_i)_{(k)} - (\mathcal{B}_i)_{(k)}\mathcal{W}_t^{(k)}(\mathcal{C}_i)_{(k)}^T)
      \]
      
   f) The mode-k operator $\mathcal{S}_k$ remains high rank because matricization along d scrambles mode-k structure
   
   g) Specifically, if $rank(\{\mathcal{F}_i\}) = N'$:
      \[
      sr_k(\mathcal{G}_t) \geq \min(N_k/2, N')
      \]

2) Now for \textbf{Tensor low-rank}:

   a) Each mode k is handled independently with its own projection:
      \[
      \mathcal{R}_t = \mathcal{G}_t \times_1 P_1^T \times_2 P_2^T \times_3 \cdots \times_d P_d^T
      \]

   b) From Theorem 2 (proven earlier), under SGD:
      \[
      \|(\mathcal{R}_t)_{(k)}\|_F \leq \left[1-\eta(\kappa_{t-1}^{(k)}-L_A^{(k)}-L_B^{(k)}L_C^{(k)}D_k^2)\right] \|(\mathcal{R}_{t-1})_{(k)}\|_F
      \]

   c) From Corollary 2, for each mode k:
      \[
      sr_k(\mathcal{G}_t) \leq sr_k(\mathcal{G}_{t_0}^{\parallel}) + \left(\frac{1-\eta\lambda_2^{(k)}}{1-\eta\lambda_1^{(k)}}\right)^{2(t-t_0)} \frac{\|\mathcal{G}_0-\mathcal{G}_{t_0}^{\parallel}\|_F^2}{\|\mathcal{G}_{t_0}^{\parallel}\|_2^2}
      \]

   d) Therefore $sr_k(\mathcal{G}_t) \leq N_k/2$ for large t, for all modes k simultaneously
\end{proof}
\begin{remark}
The key insight is that matricization in GaLore fundamentally cannot preserve low-rank structure in all modes simultaneously, while the tensor approach of \textbf{Tensor low-rank} naturally handles each mode's rank structure independently and optimally.
\end{remark}

\subsection{Sparsity types}
\label{app:sparsity_types}

\end{document}